\documentclass{article}

\usepackage{arxiv}

\usepackage[utf8]{inputenc} 
\usepackage[T1]{fontenc}    
\usepackage{url}            
\usepackage{booktabs}       
\usepackage{amsfonts}       
\usepackage{nicefrac}       
\usepackage{microtype}      
\usepackage{graphicx}
\usepackage{natbib}
\usepackage{doi}

\usepackage{mathrsfs}
\usepackage{mathtools}
\usepackage{amssymb}
\usepackage{amsmath,amsfonts,bm,bbm}
\usepackage{ulem}
\usepackage{pifont} 

\usepackage{braket}
\usepackage{wrapfig}
\usepackage{comment}
\usepackage{enumerate}
\usepackage{graphicx}

\usepackage{hyperref}

\usepackage{amsthm}
\newtheorem{theorem}{Theorem}[section]
\newtheorem{corollary}{Corollary}[theorem]
\newtheorem{lemma}[theorem]{Lemma}
\newtheorem{definition}[theorem]{Definition}
\newtheorem{assumption}[theorem]{Assumption}
\newtheorem{remark}[theorem]{Remark}
\newtheorem{proposition}[theorem]{Proposition}
\newenvironment{sketchproof}{\noindent\textit{Proof Idea:}}{\hfill$\square$}

\usepackage{tikz}
\usetikzlibrary{graphs,quotes}
\usepackage{tikz-cd}
\usetikzlibrary{shapes}

\usepackage{math_symbols}   

\title{Characterizing Overfitting in Kernel Ridgeless Regression Through the Eigenspectrum}

\author{%
  Tin Sum Cheng  \quad Aurelien Lucchi \\
  Department of Mathematics and Computer Science\\
  University of Basel\\
  \texttt{\{tinsum.cheng, aurelien.lucchi\}@unibas.ch} \\
  \And
 Anastasis Kratsios\\
 Department of Mathematics and Statistics, \\
 McMaster University and Vector Institute, \\ 
 \texttt{kratsioa@mcmaster.ca}
 \And
 David Belius\\
Faculty of Mathematics and Computer Science\\
UniDistance Suisse\\
\texttt{david.belius@cantab.ch}
}
\begin{document}
\maketitle

\begin{abstract}
We derive new bounds for the condition number of kernel matrices, which we then use to enhance existing non-asymptotic test error bounds for kernel ridgeless regression (KRR) in the over-parameterized regime for a fixed input dimension. For kernels with polynomial spectral decay, we recover the bound from previous work; for exponential decay, our bound is non-trivial and novel.
Our contribution is two-fold: (i) we rigorously prove the phenomena of tempered overfitting and catastrophic overfitting under the sub-Gaussian design assumption, closing an existing gap in the literature; (ii) we identify that the independence of the features plays an important role in guaranteeing tempered overfitting, raising concerns about approximating KRR generalization using the Gaussian design assumption in previous literature.
\end{abstract}
\section{Introduction} \label{section:introduction}
\begin{figure}[th]
    \centering
    \includegraphics[width = 0.48\columnwidth]{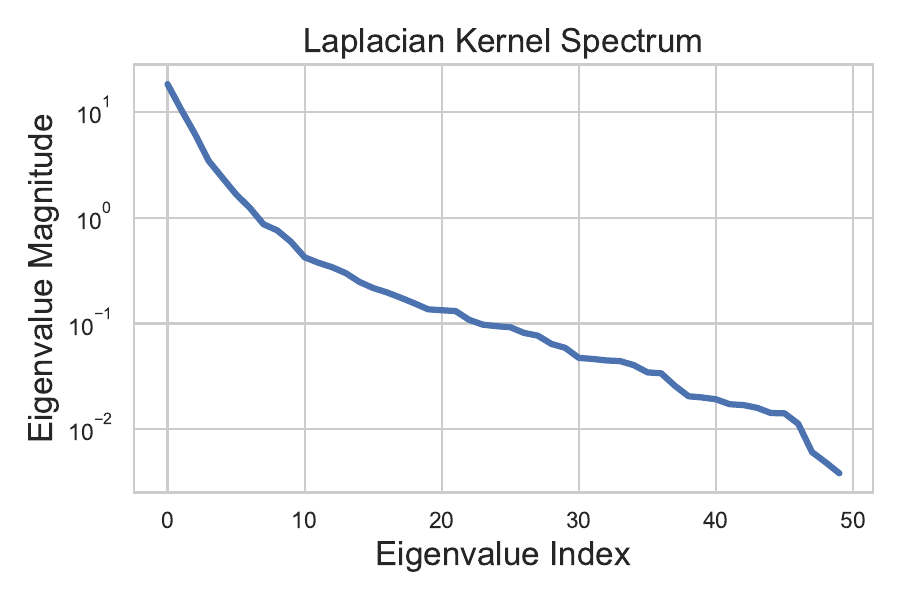}
    \includegraphics[width = 0.48\columnwidth]{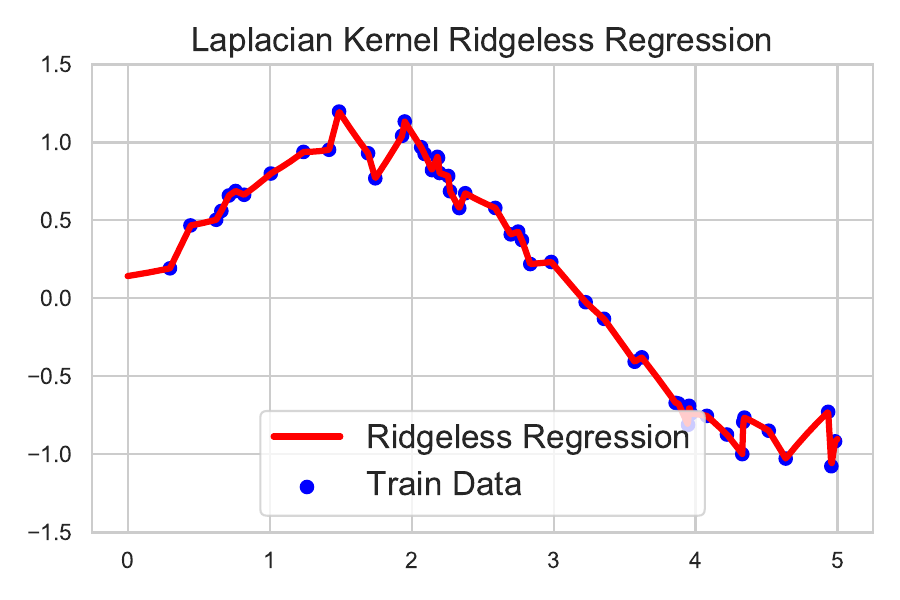}
        \includegraphics[width = 0.48\columnwidth]{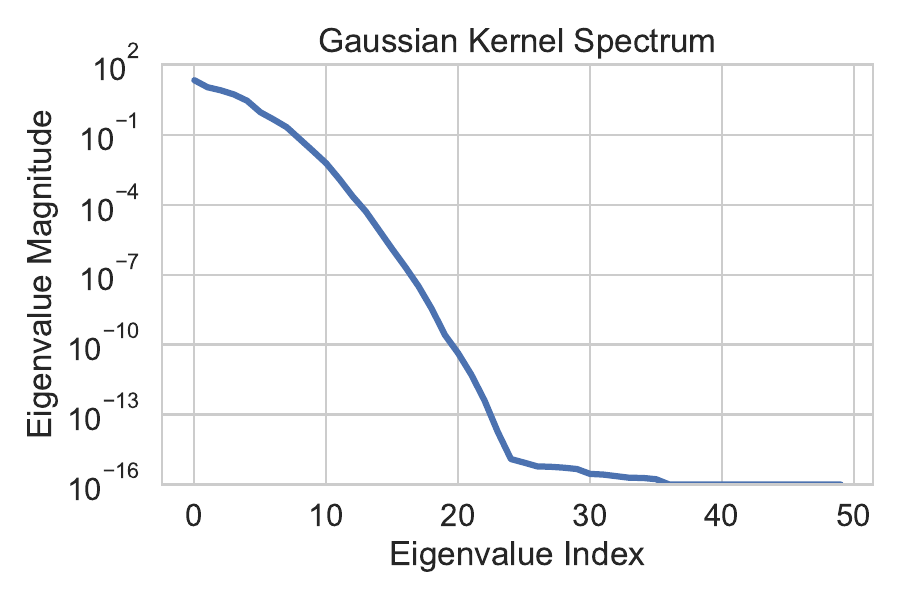}
    \includegraphics[width = 0.48\columnwidth]{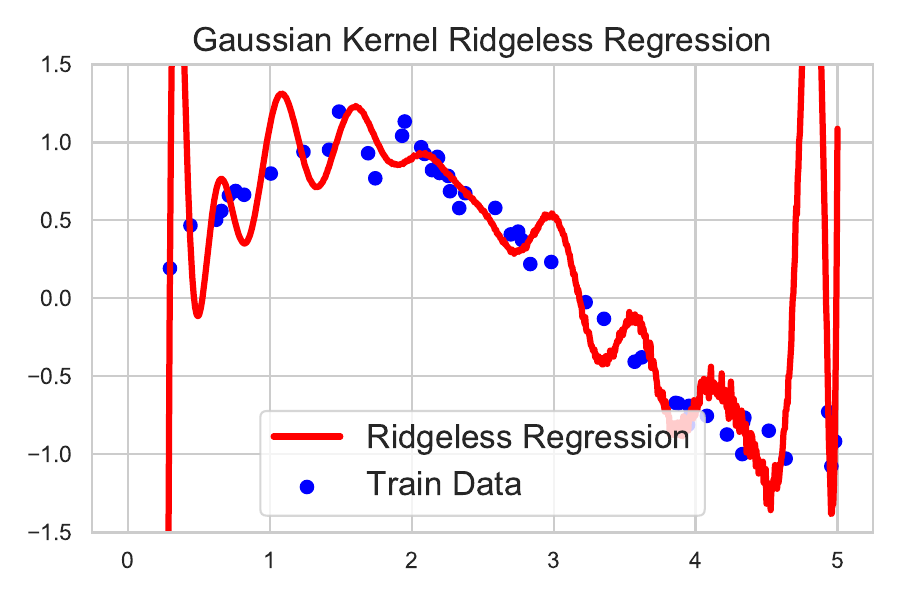}
    \caption{Kernel spectra for Laplacian and Gaussian kernels and their overfitting behaviours.  
    \hfill\\
    \textbf{Tempered Overfitting}: The empirical kernel spectrum of the Laplacian kernel decays moderately (top left), and so does the quality of its test-set performance as one departs from the training data (top right). 
    \hfill\\
    \textbf{Catastrophic Overfitting}: The Gaussian kernel exhibits rapid spectral decay (bottom left), and so does the reliability of its test-set performance for inputs far from the training data (bottom right). 
    }
    \label{fig:kernel_example}
\end{figure}

Kernel ridge regression (KRR) plays a pivotal role in machine learning since it offers an expressive and rapidly trainable framework for modeling complex relationships in data. In recent years, kernels have regained significance in deep learning theory since many deep neural networks (DNNs) can be understood as converging to certain kernel limits. 

Its significance has been underscored by its ability to approximate deep neural network (DNN) training under certain conditions, providing a tractable avenue for analytical exploration of test error and robust theoretical guarantees \cite{jacot2018neural,arora2019fine,bordelon2020spectrum}. The adaptability of kernel regression positions it as a crucial tool in various machine learning applications, making it imperative to comprehensively understand its behavior, particularly concerning overfitting.

Despite the increasing attention directed towards kernel ridge regression, the existing literature predominantly concentrates on overfitting phenomena in either the high input dimensional regime or the asymptotic regime \cite{liang2020just,mei2022generalization,misiakiewicz2022spectrum}, also known as the ultra-high dimensional regime \cite{zou2009adaptive,FanSamwothWu_2009_UltraHighDim}. Notably, the focus on asymptotic bounds, requiring the input dimension to approach infinity, may not align with the finite nature of real-world datasets and target functions. Similarly, classical Rademacher-based bounds, e.g.~\cite{BartlettMendelson_2002_JMLR_GausRadComplex}, require that the weights of the kernel regressor satisfy data-independent bounds, a restriction that is also not implemented in standard kernel ridge regression algorithms.  These mismatches between idealized mathematical assumptions and practical implementation standards necessitate a more nuanced exploration of overfitting in kernel regression in a fixed input dimension.
s
\paragraph{Contributions}
This work aims to understand the overfitting behaviour for kernel ridge regression (KRR).
Our main contributions are summarized as follows:
\begin{enumerate}
    \item We rigorously derive \textit{tight} non-asymptotic upper and lower bounds for the test error of the minimum norm interpolant under a sub-Gaussian design assumption on the features. While this assumption assumes independence of the features, this point will be relaxed in contribution \#3.
    \item Consequently, we show that a polynomially decaying spectrum yields \textit{tempered overfitting} (Theorem \ref{theorem:main:2}), whereas an exponentially decaying spectrum leads to \textit{catastrophic overfitting} (Theorem \ref{theorem:main:3}), filling a gap in the existing literature.
    \item We extend our analysis to the case of sub-Gaussian but possibly dependent features. We discover a qualitative difference in overfitting behavior that was previously unknown in the literature using the Gaussian design assumption to approximate KRR test error. This raises concerns that previous literature may have oversimplified the KRR setting by relying on the Gaussian design assumption.
\end{enumerate}

\paragraph{Motivation}
This paper is motivated by observations made in \cite{mallinar2022benign}, where the Laplacian kernel (with a polynomial spectrum) does not suffer from catastrophic overfitting even without ridge regularization, whereas the Gaussian kernel (with an exponential spectrum) does. The correspondence between polynomial and exponential spectral decay rates and the tempered and catastrophic overfitting regimes is illustrated in Figure \ref{fig:kernel_example}. However, \cite{mallinar2022benign} relied on findings from \cite{simon2021eigenlearning}, which inevitably depend on the Gaussian design assumption. We aim to explore whether it is possible to characterize overfitting behavior solely based on the kernel eigen-spectrum under a weaker assumption. The first step, which is undertaken in this paper, is to relax the assumption to sub-Gaussian.

\paragraph{Organization of the Paper}
The structure of this paper is as follows:
\begin{enumerate}
\item In Section \ref{section:previous_work}, we discuss how our work differs from previous studies and complements their results. A summary for comparison can be found in Table \ref{table:comparison}.
\item In Section \ref{section:setting}, we state the definitions and assumptions for this paper.
\item In Section \ref{section:main_result}, we present our main results (Theorems \ref{theorem:main:1}, \ref{theorem:main:2}, and \ref{theorem:main:3}) and interpret their significance, novelty, and improvement compared to previous work. 
\item In Section \ref{section:experiments}, we showcase the empirical results of a simple experiment to validate our findings. 
\item In Section \ref{section:discussion}, we discuss the implications of our contributions in-depth, including their limitations and potential directions for future research.

\item In Section \ref{section:sub_gaussianity}, we present our proof under the Sub-Gaussian design assumption \ref{assumption:sub_gaussian_design}.
\item In Section \ref{section:technical_lemmata}, we list the technical lemmata used in this paper.
\end{enumerate}
\section{Previous Work} \label{section:previous_work}
Traditional statistical wisdom has influenced classical machine learning models to focus on mitigating overfitting with the belief that doing so maximizes the ability of a model to generalize beyond the training data.  However, these traditional ideas have been challenged by the discovery of the ``benign overfitting'' phenomenon, see e.g~\cite{liang2020just, bartlett2020benign, tsigler2023benign, haas2023mind}, in the context of KRR.  
A key factor is that traditional statistics operate in the under-parameterized setting where the number of training instances exceeds the number of parameters.  This assumption is rarely applicable to modern machine learning, where models depend on vastly more parameters than their training instances, and thus, classical statistical thought no longer applies.

\subsection{(Sub-)Gaussian Design Assumption} \label{subsection:gaussian_assumption}

Many previous works \cite{jacot2020implicit, bordelon2020spectrum, simon2021eigenlearning, loureiro2021learning, cui2021generalization} require the so-called Gaussian design assumption, where isotropic kernel feature vectors are replaced by Gaussian vectors, to prove their results on KRR generalization. (See Assumption \ref{assumption:gaussian_design} in Section \ref{section:setting} for details.) In contrast, we obtain tight bounds on test error under the weaker sub-Gaussian design assumption, where we replace the isotropic kernel feature vectors with sub-Gaussian vectors. This seemingly simple extension yields surprisingly many fundamental differences compared to previous work:
\begin{enumerate}
    \item Generally, a random vector $\z=(z_k)_{k=1}^p\in\R^p$ is isotropic if $\Expect{}{\z\z^\top}=\I_p$, meaning the entries are uncorrelated but possibly dependent on each other. However, for a Gaussian vector, uncorrelatedness implies independence. Therefore, the Gaussian design assumption implicitly requires the independence of features, which, as we will show later, is crucial to the phenomenon of tempered overfitting with a polynomially decaying spectrum.
    \item Also, the argument of many previous works relies heavily on the Gaussian design assumption, which cannot be generalized to sub-Gaussian design by any universality argument. For instance, \cite{simon2021eigenlearning} utilizes rotational invariant property of Gaussian vectors; \cite{bordelon2020spectrum, loureiro2021learning, cui2021generalization} inevitably require the Gaussian design assumption in the Replica method. By relaxing to the sub-Gaussian assumption, we show that many nice properties of Gaussian vectors, such as rotational invariant property, smoothness/continuity, anti-concentration, are not important to the overfitting behaviour of the ridgeless regression, extending previous results to a more general setting.
    \item Last but not least, both Gaussian and sub-Gaussian design assumptions require the feature dimension $M$ to be finite, while it should be infinite in the case of a kernel. For the sake of completeness, we provide a result showing that the overfitting behavior of an infinite rank kernel can be approximated by its finite rank truncation. See Proposition \ref{proposition:finite_rank} for details.
\end{enumerate}

\subsection{Test Error on Ridgeless Regression} \label{subsection:ridgeless_regression}
Many previous works \cite{arora2019fine, liang2020just, bordelon2020spectrum, bartlett2020benign, simon2021eigenlearning, mei2021generalization, misiakiewicz2022spectrum, bach2023high, cheng2023theoretical} are devoted to bounding the KRR test error in different settings. In the context of benign overfitting,
a recent related paper \cite{tsigler2023benign} gives tight non-asymptotic bounds on the ridgeless regression test error  \textbf{\textit{under the assumption that the condition number of kernel matrix is bounded by some constant}}. 
Our random matrix theoretic arguments successfully allow us to derive tight non-asymptotic bounds for the condition number of the empirical kernel matrix (see Theorem \ref{theorem:main:1}) and to apply some of their technical tools without their stylized assumptions.  

\subsection{Overfitting} \label{subsection:overfitting}

Recently, \cite{mallinar2022benign} characterized previous results on overfitting, especially in the context of KRR, and classified them into three categories
1) \textit{benign overfitting} meaning that the learned model interpolates the noisy training data while exhibiting a negligible reduction in test performance decline, 
2) \textit{tempered overfitting}, which happens when the learned model exhibits a bounded reduction in test set performance due resulting from an interpolation of the training data, 
and 3) \textit{catastrophic overfitting} which covers the case where the test error is unbounded due to the learned model having interpolated the training data. 
In this paper, we characterize the tempered and catastrophic overfitting cases, omitting benign overfitting which has been characterized in prior work. According to \cite{mallinar2022benign}, even with Gaussian design assumption, benign overfitting occurs only when the spectral eigen-decay is slower than any polynomial decay $\lambda_k=\bigtheta{}{k^{-1-\epsilon}}$ for any constant $\epsilon>0$, for instance the linear-poly-logarithmic decay $\lambda_k=\bigtheta{}{k^{-1}\log^{-a}(k)}$ for some constant $a>0$. Such spectral eigen-decay, to the best of our knowledge, does not appear in commonly-known kernels.

\subsection{Comparison to other Results} \label{subsection:concurrent_work}
A comparison of our results to the state-of-the-art in the literature is detailed in Table~\ref{table:comparison}. Especially relevant is the comparison to \cite{barzilai2023generalization}. We note that this paper is in fact a concurrent work, as it was published on arXiv just four weeks prior to the submission deadline for ICML. The strength of \cite{barzilai2023generalization} is the general setting under which they perform their analysis. However, our analysis yields tighter bounds than theirs for the class of kernels to which our analysis applies, achieved via tighter bounds on the involved kernel eigenspectrum. Importantly, unlike their results, our analysis provides upper and matching lower bounds on the test error. Additionally, we address the catastrophic behavior with exponential eigen-decay, which they have not considered

\newcommand{\tick}{{\color{green}{\boldsymbol{\checkmark}}}}
\newcommand{\cross}{{\color{red}{\boldsymbol{\times}}}}
\begin{table*}[ht]
\caption{Comparison with prior works}
\label{table:comparison}
\begin{center}
\resizebox{\textwidth}{!}{
\begin{tabular}{ccccc}
\toprule
    & \cite{mallinar2022benign} & \cite{tsigler2023benign} & \cite{barzilai2023generalization} & \textbf{This paper}  \\
\midrule
Assumption on kernel  & Gaussian feature & Bound on condition number & Kernel cont. and bdd. & Sub-Gaussian feature \\
Non-asymptotic bounds   & $\cross$ & $\tick$ & $\tick$ & $\tick$  \\
Overfitting for poly. decay  & $\tick$ & $\cross$ & $\tick$ & $\tick$ \\
Overfitting for exp. decay   & $\tick$  & $\cross$ & $\cross$ & $\tick$ \\
\bottomrule
\end{tabular}
}
\end{center}
\vskip -0.1in
\end{table*}

\section{Setting} \label{section:setting}
Given a kernel $K$ with reproducing kernel Hilbert space (RKHS) $\mathcal{H}$, we consider the kernel ridge regression (KRR) problem:
\begin{equation*}
    \min_{f\in\mathcal{H}} 
    \,
    \sum_{i=1}^N \left(f(x_i)-y_i\right)^2 + \lambda \|f\|_\mathcal{H}^2.
\end{equation*}
The solution $\hat{f}$ to the KRR problem, called the kernel ridge regressor, is unique whenever $\lambda>0$. For $\lambda=0$ and $\text{dim}(\mathcal{H})>N$, with minor abuse of notation, we write $\hat{f}$ the norm-minimizing interpolant:
\begin{equation*}
    \hat{f} \in \argmin_{f(x_i)=y_i,\forall i} \|f\|_\mathcal{H}. 
\end{equation*}
Given a data-distribution $\mu$ on the input space $\mathcal{X}$, using the Mercer theorem we decompose:
\begin{equation*}
    K(x,x')=\sum_{k=1}^M \lambda_k\psi_k(x)\psi_k(x'),
\end{equation*}
where $M\in\N\cup\{\infty\}$ is the kernel rank, $\lambda_k$'s are the eigenvalues indexed in decreasing order with corresponding eigenfunctions $\psi_k$'s. 
Hence the (random) kernel matrix $\mathbf{K}=[K(x_i,x_j)]_{i,j}$ can be written concisely in matrix form
\begin{equation*}
    \mathbf{K} = \bm{\Psi}^\top  \bm{\Lambda}\bm{\Psi},
\end{equation*}
where $\bm{\Psi}=[\psi_k(x_i)]\in\R^{M\times N}$ is the design block. 

Next, we introduce two important assumptions in this paper.
\begin{assumption}[Interpolation] \label{assumption:interpolation}
    Assume $M\in\N$ and there exists an integer constant $\eta>1$ (to be determined) such that $M\geq\eta N$. Also, we assume that $\lambda=0$ and hence $\hat{f}$ denotes the norm-minimizing interpolant.
\end{assumption}
\begin{remark}
    Note that Assumption \ref{assumption:interpolation} only requires the feature dimension $M$ to be larger than the threshold $\eta N$, and $M$ does not necessarily need to be linear in $N$. This is different from the so-called proportional regime in \cite{liang2020just, liu2020kernel}, where the proportion $\frac{M}{N}$ converges to some constant $\gamma>0$.
\end{remark}
Next, we assume sub-Gaussianity of the eigenfunctions, which is very standard in KRR literature (to name a few, \cite{liang2020just, bartlett2020benign, tsigler2023benign, bach2023high}):
\begin{assumption}[Sub-Gaussian design] \label{assumption:sub_gaussian_design}
    Let $M\in\N$. For every $k=1,...,M$, the random variable $\psi_k(x)$ is replaced by an independent sub-Gaussian variable with uniformly bounded sub-Gaussian norm.
\end{assumption}
This is a relaxation of the Gaussian design assumption, which is used in \cite{bordelon2020spectrum, cui2021generalization, loureiro2021learning, simon2021eigenlearning}:
\begin{assumption}[Gaussian design] \label{assumption:gaussian_design}
    Let $M\in\N$. For every $k=1,...,M$, the random variable $\psi_k(x)$ is replaced by an independent standard Gaussian variable.
\end{assumption}
Under Assumption~\ref{assumption:gaussian_design}, the learning task is simply linear regression with the feature vectors $\psi_k(x)$'s replaced by $M$-dimensional Gaussian inputs $\bm{\phi}_k\eqdef  \bm{\Lambda}^{1/2}\bm{\psi}_k\sim\mathcal{N}(0,\bm{\Lambda}^{1/2})$ for all $k$. 
\section{Main Result} \label{section:main_result}
The analysis of our main result consists of three steps. First, we bound the condition number of the kernel matrix $\mathbf{K}$ under the interpolation assumption (Assumption \ref{assumption:interpolation}) and sub-Gaussian design assumption (Assumption (Assumption \ref{assumption:sub_gaussian_design})) in Theorem \ref{theorem:main:1}. Next, we use this result to give a tight bound of the test error and conclude the effect of the spectral decay on overfitting in Theorem \ref{theorem:main:2}. Lastly, we demonstrate the necessity of feature independence by Theorem \ref{theorem:main:3}. The formal versions of the main theorems can be found in Section \ref{section:sub_gaussianity}.
\subsection{Condition Number}
First, we show that the condition number of the kernel matrix is bounded with polynomial and exponential decays differently.
\begin{theorem}[Bounding the Condition Number]\label{theorem:main:1}
    Suppose $M,N\in\N$ such that $M\geq \eta N$ for some constant $\eta>1$ that is large enough. Let $\Psi\in\R^{M\times N}$ be a matrix with i.i.d. isotropic random vectors $\Psi_i$'s with independent sub-Gaussian entries as columns. Let $\La=\diag(\lambda_k)_{k=1}^M\in\R^{M\times M}$ be a diagonal matrix. Then with high probability, the condition number $\frac{s_{\max}(\K)}{s_{\min}(\K)}$ of the matrix $\K=\Ps^\top\La\Ps\in\R^{M\times N}$ is bounded by:
    \begin{enumerate}
        \item 
        $
            \frac{s_{\max}(\K)}{s_{\min}(\K)}
            =
            \bigo{N}{\frac{\lambda_1}{\lambda_N}}, 
        $
        if $\lambda_k$'s decay polynomially;
        \item 
        $
            \frac{s_{\max}(\K)}{s_{\min}(\K)}
            =
            \bigtheta{N}{\frac{\lambda_1}{\lambda_N}N}, 
        $
        if $\lambda_k$'s decay exponentially and furthermore $\Psi\in\R^{M\times N}$ is a Gaussian random matrix with $M=\eta N$.
    \end{enumerate}
\end{theorem}
\begin{sketchproof}
    It is well known that, with high probability, $s_{\max}(\K)\asymp N\lambda_1$ for both types of eigen-decay. The major difference between the polynomial and exponential decay is the lower bound of $s_{\min}$. In the former case, we apply random matrix concentration from \cite{vershynin2010introduction} to obtain $s_{\min}(\K)\gtrsim N\lambda_N$; in the latter case, we apply a lemma from \cite{tao2012topics} and the anti-concentration of Gaussian to obtain $s_{\min}(\K)\asymp N\lambda_N$. The full proof can be found in Lemmata \ref{lemma:condition_number:poly} and Lemmata \ref{lemma:condition_number:exp} in the appendix.
\end{sketchproof}
Intuitively, one might suppose that $\frac{s_{\max}(\K)}{s_{\min}(\K)}\approx\frac{\lambda_1}{\lambda_N}$.
From Theorem \ref{theorem:main:1} and the experiments in Section \ref{section:experiments}, we can see that polynomial spectral eigen-decays yields the intuitive bound  $\frac{s_{\max}(\K)}{s_{\min}(\K)}\approx\frac{\lambda_1}{\lambda_N}$ on the condition number. In the latter case, however, the intuitive bound does not hold even if we restrict the feature dimension $M=\eta N$ and the random matrix $\Ps$ to Gaussian. 
\\
While we apply Theorem \ref{theorem:main:1} to bound the test error in the rest of the paper, the bound of the condition number can be of independent interest, for instance when studying the convergence properties of gradient-based methods.

\subsection{Classifying Overfitting Regimes}
The bound on the condition number of the kernel matrix in Theorem \ref{theorem:main:1} can be applied to bound the test error of kernel ridgeless regression under the interpolation assumption (Assumption \ref{assumption:interpolation}) and sub-Gaussian design assumption (Assumption \ref{assumption:sub_gaussian_design}). 
To this end, we formally define the test error as follows: \\
Let $y_i=f^{\star}(x_i)+\epsilon_i$ for all $i=1,...,N$, where $f^{\star}\in\mathcal{H}$ is the target function and $\epsilon_i$'s are draws from a centered sub-Gaussian random variable $\epsilon$ with variance $\Expect{}{\epsilon^2}=\sigma^2>0$. 
\\
We define the test error (or excess risk) $\mathcal{R}$ to be the mean square error (MSE) between the target function $f^{\star}$ and the norm-minimizing interpolant $\hat{f}$ of a given fixed dataset of size $N$ averaging out the noise in the dataset:
\begin{equation}
    \mathcal{R}\eqdef \Expect{x,\epsilon}{(f^{\star}(x)-\hat{f}(x))^2}.
\end{equation}
Now, we present our result on overfitting with polynomial and exponential decays. 

\begin{theorem}[Overfitting with Polynomial and Exponential Eigen-Decay] \label{theorem:main:2}
    Suppose the interpolation assumption (Assumption \ref{assumption:interpolation}) and the sub-Gaussian design assumption (Assumption (Assumption \ref{assumption:sub_gaussian_design})) hold.
    \footnote{We suppose Assumption (Assumption (Assumption \ref{assumption:sub_gaussian_design})) to hold in the sense that the distributions of the regressor $\hat{f}(x)=\K_x^\top\K^{-1}\y$ and the target function $f^*(x)$ evaluated on a random test point $x$ are replaced by those with random variables $\ps^\top\La^{1/2}(\Ps^\top\La\Ps)^{-1}(\Ps^\top\th^* + \ep)$ and $\ps^\top\th^*$ for some fixed target vector $\th^*\in\R^M$, noise vector $\ep\in\R^N$ and i.i.d. random vectors $\ps$, $\Ps_i$'s.}
    Then there exists a constant $C\in(0,1)$ independent of $M,N$, such that with high probability, the followings hold:
    \begin{enumerate}
        \item if $\lambda_k$'s decays polynomially, then $C\leq\mathcal{R}\leq C^{-1}$. In other words, a kernel with polynomial decay exhibits tempered overfitting.
        \item if $\lambda_k$'s decays exponentially, then $\mathcal{R}\geq C N$. In other words, a kernel with polynomial decay exhibits catastrophic overfitting.
    \end{enumerate}
\end{theorem}
\begin{sketchproof}
    The proof proceeds similarly to \cite{tsigler2023benign}, where we first use the upper bound on the condition number of the kernel matrix in Theorem \ref{theorem:main:1} together with the result from \cite{tsigler2023benign} to bound the KRR test error from above. Then we apply the result of the matching lower bound from Theorem \cite{tsigler2023benign} to conclude the statement. The full proof can be found in Corollary \ref{corollary:tempered} and Theorem \ref{theorem:catastrophic} in the appendix.
\end{sketchproof}
\paragraph{Extension of Previous Results}
Although the two results in Theorem \ref{theorem:main:2} are not new in the literature, we have provided a qualitatively better analysis:
1) the upper bound of the test error of Theorem \ref{theorem:main:2} is a result we can recover from \citep[Theorem 2]{barzilai2023generalization}, but our probability is of exponential decay which is faster than the Markov type bound of \cite{barzilai2023generalization}; 2) the tempered overfitting behaviour of kernel with polynomial decay, that is reported in \cite{cui2021generalization,simon2021eigenlearning,mallinar2022benign}, which used the Gaussian design Assumption \ref{assumption:gaussian_design}. We replace this with the more general sub-Gaussian design assumption (Assumption (Assumption \ref{assumption:sub_gaussian_design})).

\paragraph{Benign Overfitting}
As reported in \cite{bartlett2020benign, mallinar2022benign, barzilai2023generalization}, if the spectral eigen-decay is much slower than polynomial decay, say $\lambda_k = \bigtheta{k}{k^{-1}\log^{-a} k}$ for some $a > 0$, then the overfitting is benign under the Gaussian design assumption. However, to the best of our knowledge, there is no natural kernel exhibiting such eigen-decay $\lambda_k = \Theta(k^{-1}\log^{-a} k)$. Thus, we only consider polynomial and exponential eigen-decays, which represent realistic scenarios as in Figure \ref{fig:overfitting}. Hence, by Theorem \ref{theorem:main:2}, the discussion of benign overfitting is out of the scope of this paper.
\subsection{Independent versus Dependent Features}
We further investigate the reason behind tempered overfitting with polynomial eigen-decay and discover that the independence between the entries of the feature vector $\ps$ plays an important role in bounding the smallest singular value $s_{\min}(\K)$ of the kernel matrix $K$. Specifically, we have
\begin{theorem}[Smallest singular value with dependent features] \label{theorem:main:3}
    Suppose $M,N\in\N$ such that $M\geq \eta N$ for some constant $\eta>1$ large enough. Let $\Psi\in\R^{M\times N}$ be a matrix with i.i.d. isotropic random vectors $\Psi_i$'s with (possibly dependent) sub-Gaussian entries as columns. Let $\La=\diag(\lambda_k)_{k=1}^M\in\R^{M\times M}$ be a diagonal matrix with $\lambda_k$'s decaying polynomially. Then with high probability, we have
    \begin{equation*}
        s_{\min}(\K) \geq P N\lambda_N,
    \end{equation*}
    where $P$ is a positive random variable depending on $\Ps$. If the entries of each $\Psi_i$ are furthermore independent of each other, that is, if the sub-Gaussian design assumption (Assumption \ref{assumption:sub_gaussian_design}) holds, then there exists a constant $C>0$, such that with high probability, $P\geq C$, recovering the result in Theorem \ref{theorem:main:1}.
\end{theorem}
\begin{sketchproof}
    The argument follows from random matrix concentration from \cite{vershynin2010introduction}. The full proof can be found in Lemma \ref{lemma:s_min:lb:poly:sub_gaussian} in the appendix.
\end{sketchproof}

In general, when the features are dependent on each other, the random variable $P$ can vanish to zero as $N\to\infty$, rendering the lower bound in Theorem \ref{theorem:main:3} vacuous. Hence the argument in Theorem \ref{theorem:main:2} would break down in the general feature case (for instance when considering kernel feature vectors $\ps=(\psi_k(x))_{k=1}^M$, where the eigenfunctions $\psi_k$'s are generally dependent on each other).
Indeed, we discover both tempered and catastrophic overfitting phenomena can occur for kernels with polynomial eigen-decay and dependent feature vectors (see Figure \ref{fig:overfitting}). 
\begin{figure}[ht]
    \centering
    \includegraphics[width = 0.48\linewidth]{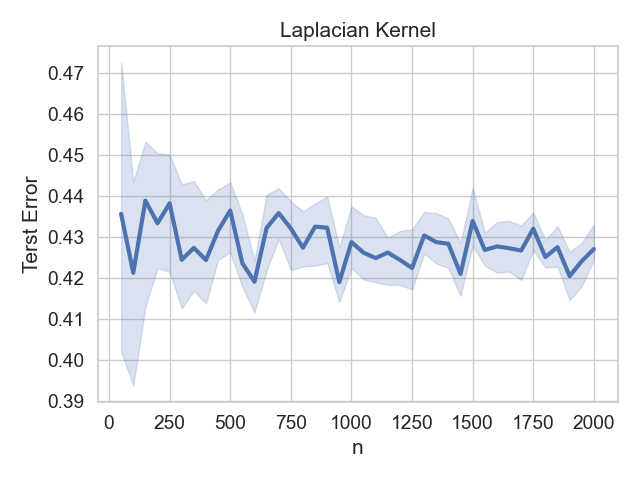}
    \includegraphics[width = 0.48\linewidth]{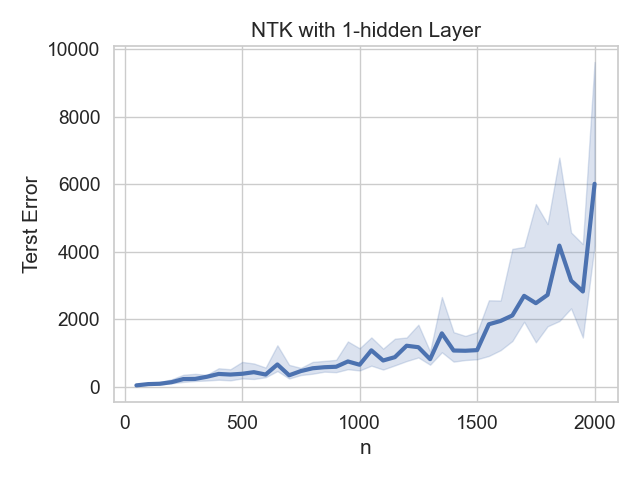}
    \caption{
    \textit{Test error of kernel interpolation on the unit 2-disk against the sample size $N$.}
    (left): Laplacian kernel $K(x,z)=e^{-\norm{x-z}{2}}$
    (right): ReLU Neural tangent kernel (NTK) for a 1-hidden layer network
    }
    \label{fig:overfitting}
\end{figure}
The neural tangent kernel (NTK) for a 1-hidden layer network is defined to be $K(x,z)=x^\top z \kappa_0(x^\top z)+ \kappa_1(x^\top z)$ where $\kappa_0(t)\eqdef1-\frac{1}{\pi}\arccos(t)$, $\kappa_1(t)\eqdef\frac{1}{\pi}\left(  t(\pi-\arccos(t))+\sqrt{1-t^2}\right)$. According to \cite{bietti2019inductive}, the NTK exhibits polynomial eigen-decay. This counterexample suggests that the conclusion drawn in \cite{mallinar2022benign} regarding the tempered overfitting with polynomial eigen-decay might be too optimistic.
\section{Experiments} \label{section:experiments}
We run several simple experiments to validate our theoretical analysis on overfitting. 

\subsection{(Sub-)Gaussian Design}
First, we validate the main results in Section \ref{section:main_result}: 1) the bound of the condition number $\frac{s_{\max}(\K)}{s_{\min}(\K)}$ for polynomial and exponential spectra as predicted in Theorem 1; 2) the tempered and  the catastrophic) overfittings for polynomial and exponential spectra.

For simplicity, we implement the experiment following Assumption (Assumption \ref{assumption:gaussian_design}).
Let $\bm{\phi}_k\sim\mathcal{N}(0,\bm{\Lambda})$ be i.i.d. Gaussian random vector with covariance $\bm{\Lambda}=\diag\{\lambda_k\}$. Write $\bm{\Phi} \in \R^{M\times N}$ be a matrix with $k^{th}$ column $\bm{\phi}_k$.  
For each pair $N$ and $M=10N$, we run over 20 random samplings for the kernel matrix $\bm{\Phi}^\top \bm{\Phi}$.  

Figure \ref{fig:cond_num} confirms that the condition number of the kernel matrix grows as described in Theorem \ref{theorem:main:1}: with $\frac{s_{\max}}{s_{\min}} \asymp \frac{\lambda_1}{\lambda_N}$ in the case of a polynomial spectrum and $\frac{s_{\max}}{s_{\min}} \asymp \frac{N\lambda_1}{\lambda_N}$ in the case of an exponential spectrum.
\begin{figure}[ht]
    \centering
    \includegraphics[width = 0.48\columnwidth]{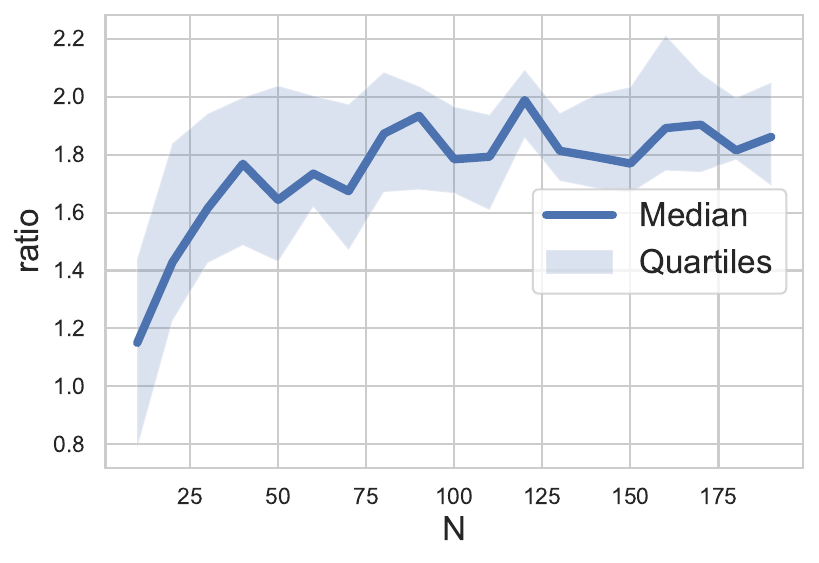}
    \includegraphics[width = 0.48\columnwidth]{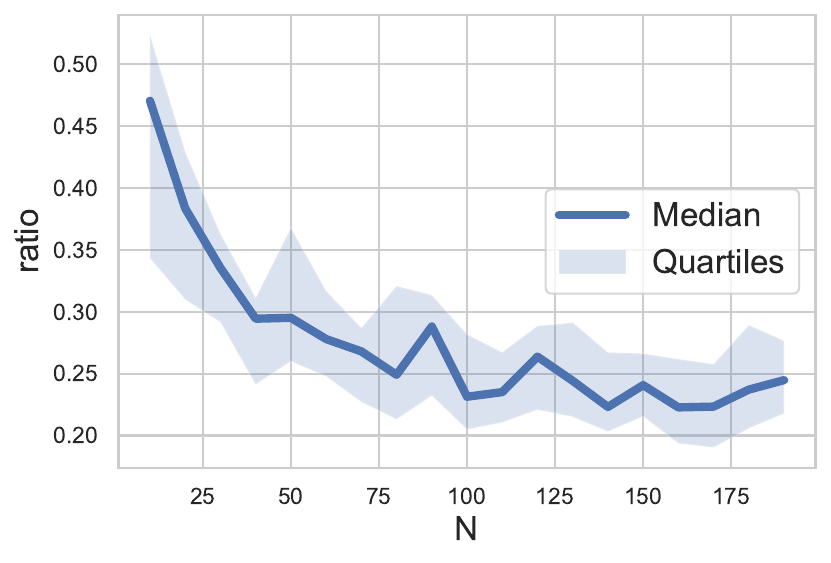}
    \caption{Validation of Theorem~\ref{theorem:main:1}: The ratios $\frac{s_{\max}}{s_{\min}} : \frac{\lambda_1}{\lambda_N}$ for the polynomial spectrum (left) and $\frac{s_{\max}}{s_{\min}} : \frac{N\lambda_1}{\lambda_N}$ for the exponential spectrum (right) are asymptotically constant.}
    \label{fig:cond_num}
\end{figure}
To compute the test error, we randomly set the true coefficient $\bm{\theta}^*\sim\mathcal{N}(0,\mathbf{I}_M)$ and let $y=(\bm{\theta}^*)^\top \bm{\phi}+\epsilon$ be the label where $\epsilon\sim\mathcal{N}(0,1)$ is the noise. We evaluate the test error using the mean square error (MSE) between the true label and the ridgeless regression on 1000 random points. For each pair $N$ and $M=10N$, we run over 20 iterations for the same true coefficient. 
In Figure \ref{fig:learning_curve}, we validate Theorem \ref{theorem:main:2}: the learning curve for polynomial decay is asymptotically bounded by constants; while that for exponential decay increases as $N\to\infty$.
\begin{figure}[ht]
    \centering
    \includegraphics[width = 0.48\columnwidth]{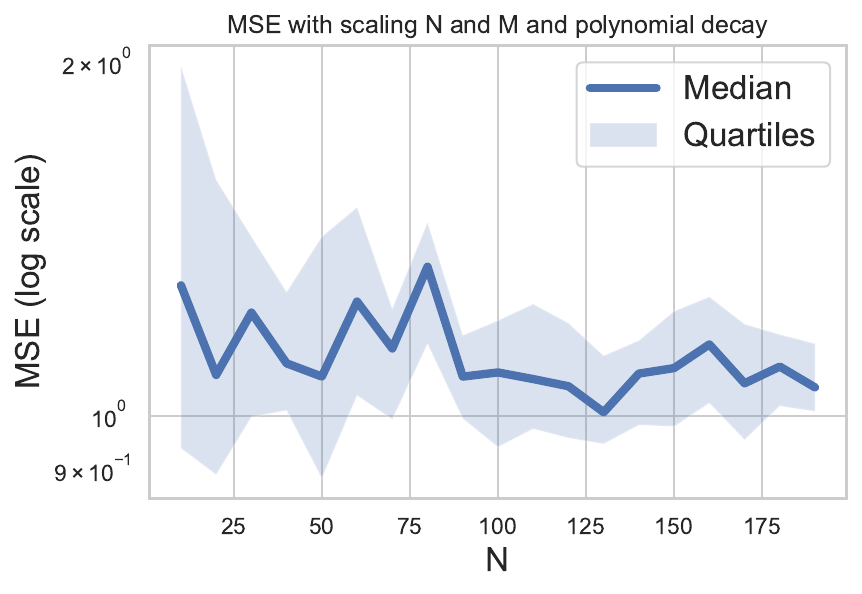}
    \includegraphics[width = 0.48\columnwidth]{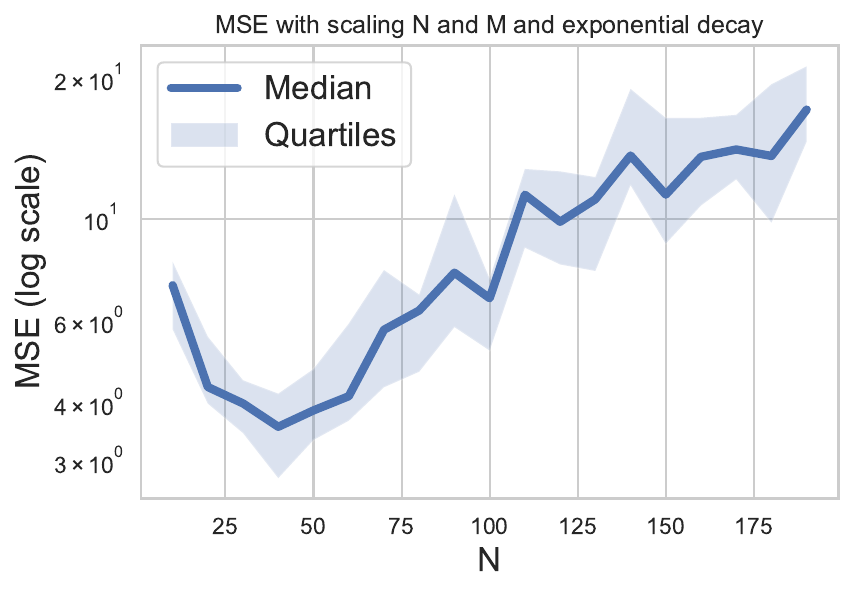}
    \caption{Validation of Theorems \ref{theorem:main:2} and \ref{theorem:main:3}: Learning curves for spectra with polynomial (left) and exponential (right) decays.}
    \label{fig:learning_curve}
\end{figure}

To validate Theorem \ref{theorem:main:1} under the sub-Gaussian design assumption (Assumption \ref{assumption:sub_gaussian_design}), we compare the empirical spectrum of $\K=\Ps^\top\La\Ps$, where the isotropic features $\Ps_i$ are either Gaussian or uniformly distributed ($\unif[-\sqrt{3},+\sqrt{3}]$). In Figure \ref{fig:sub_gaussian}, we observe that $s_{\min}\left(\frac{1}{n}\K\right)$ remains in the same magnitude as $\lambda_N$ for both types of features.
\begin{figure}[ht]
    \centering
    \includegraphics[width = 0.7\columnwidth]{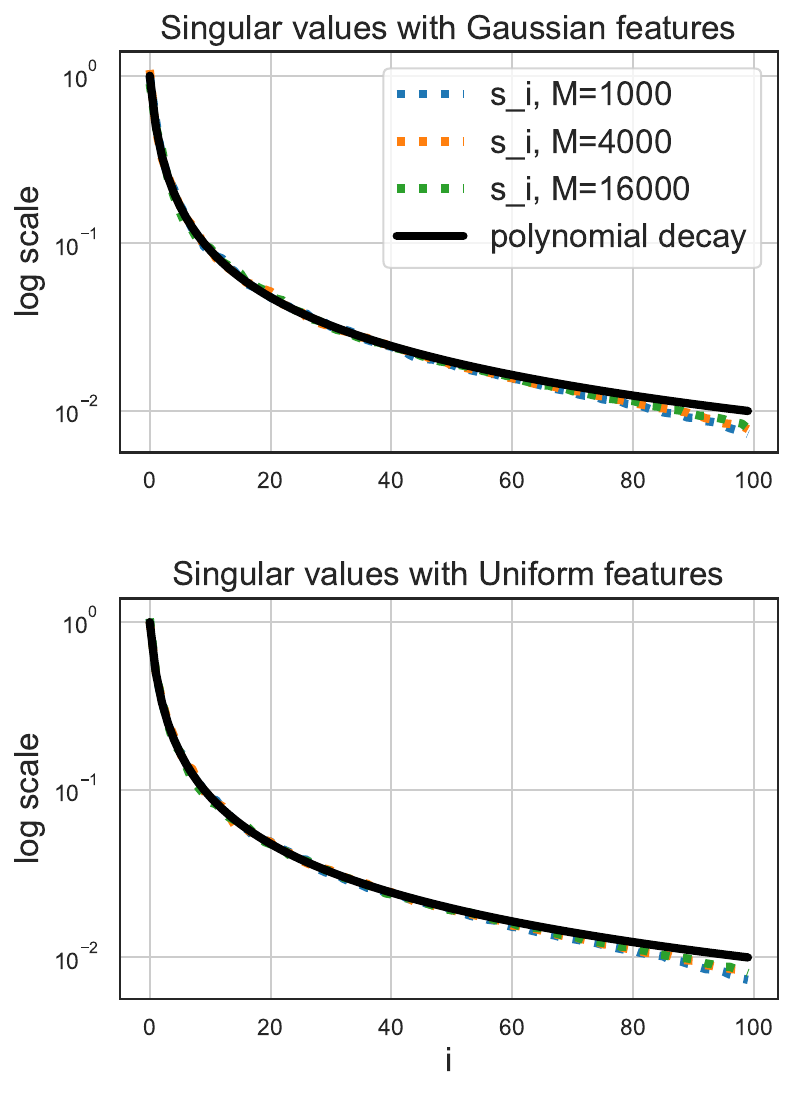}
    \caption{\textit{Empirical singular values are close to the eigen-spectrum for independent features}
    (top): The features are Gaussian $\psi\sim\mathcal{N}(0,\I_M)$.
    (bottom): The features are uniformly distributed $\psi\sim(\unif[-\sqrt{3},+\sqrt{3}])^p$.
    }
    \label{fig:sub_gaussian}
\end{figure}
To observe the effect of feature dependence on the smallest singular value $s_{\min}(\K)$, we conduct experiments with cosine and sine features. In Figure \ref{fig:co_sine}, we observe that $s_{\min}(\K)$ vanishes, thus validating our findings in Theorem \ref{theorem:main:3}.
\begin{figure}[ht]
    \centering
    \includegraphics[width = 0.7\columnwidth]{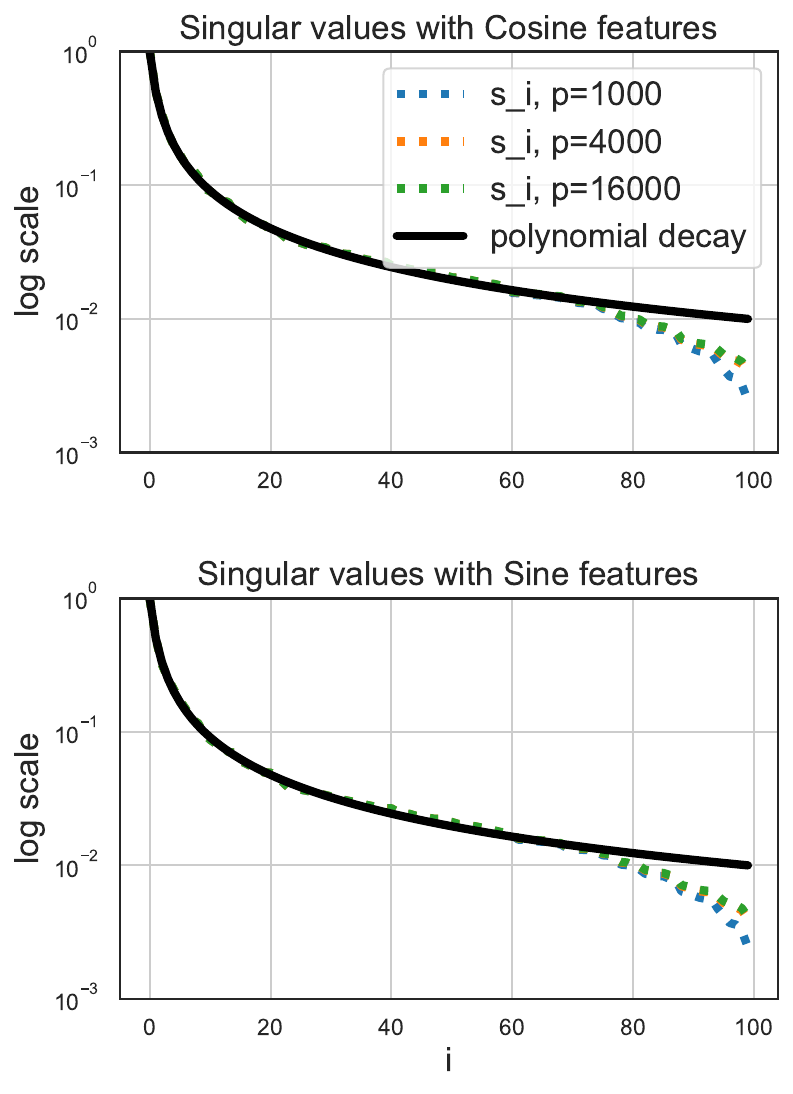}
    \caption{\textit{Empirical smallest singular value vanishes for dependent features.}
    (top): The features are cosines $\psi=(\cos(k\cdot))_{k=1}^M$.
    (bottom): The features are sines $\psi=(\sin(k\cdot))_{k=1}^M$.
}
    \label{fig:co_sine}
\end{figure}
\subsection{Dependent Feature}
Additionally, when we transition to kernels, we observe that the smallest singular value also diminishes (see Figure \ref{fig:laplacian}). In this scenario, the data follows a Gaussian distribution on the real line. Combined with the insights from Figure \ref{fig:overfitting}, where the Laplacian kernel displays tempered overfitting under different data distributions, we conclude that \textit{the condition $s_{\min}(\K)\approx \lambda_N$ is sufficient but not necessary for observing tempered overfitting in kernels with polynomial decay}.
\begin{figure}[ht]
    \centering
    \includegraphics[width = 0.48\columnwidth]{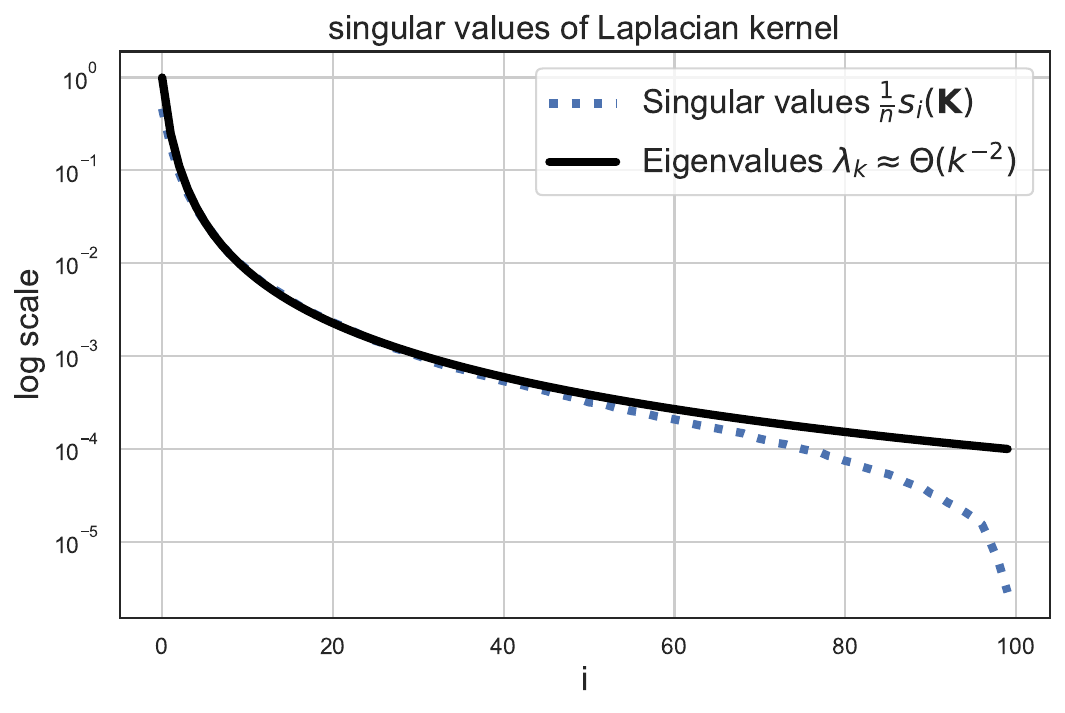}
    \includegraphics[width=0.48\columnwidth]{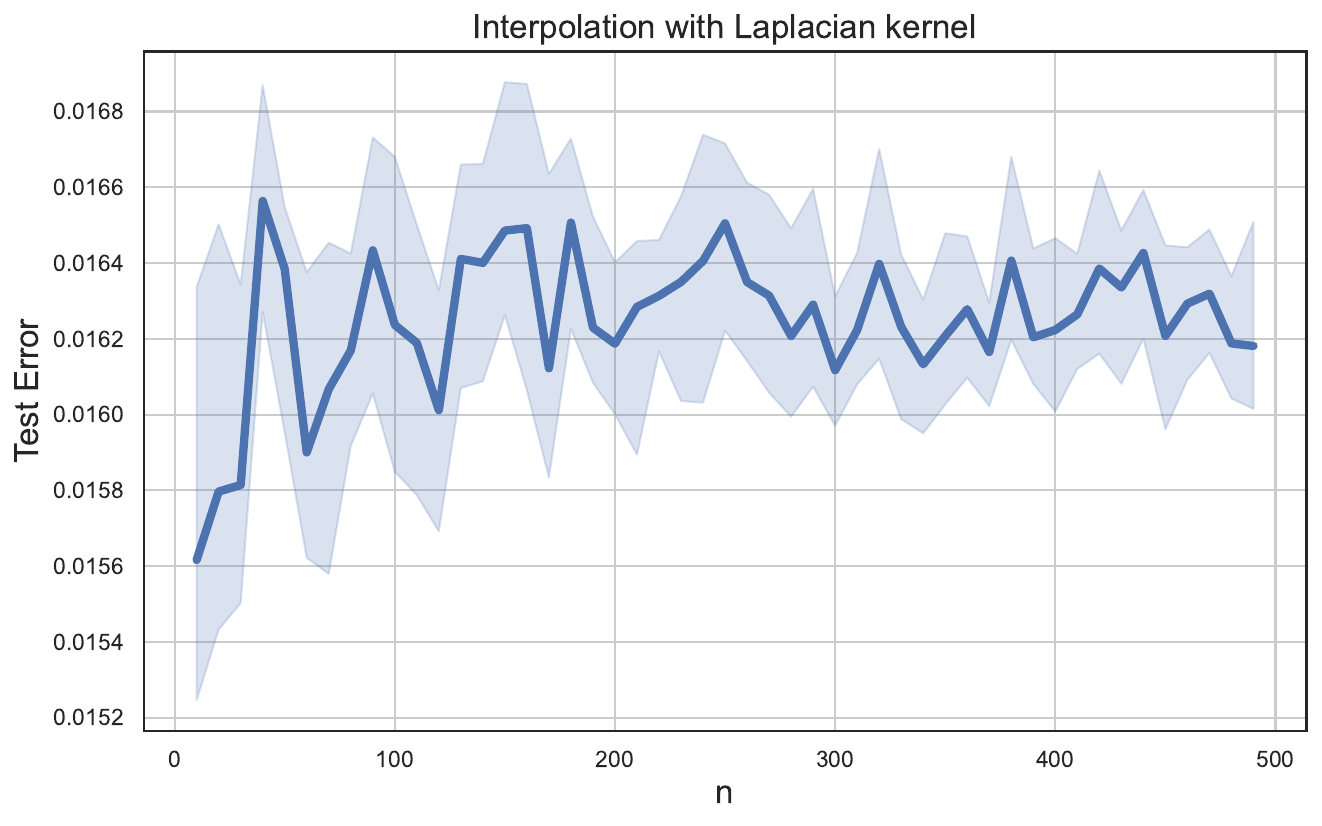}
    \caption{\textit{Interpolation with Laplacian kernel $K(x,z)=e^{-|x-z|}$ with inputs $x,z\sim\mathcal{N}(0,1)$.}
    (left): The empirical spectrum of Laplacian.
    The smallest singular value vanishes as in Figure \ref{fig:co_sine}.
    (right): The test error is bounded by some constants as $n\to\infty$, exhibiting tempered overfitting.
    }
    \label{fig:laplacian}
\end{figure}

\section{Discussion} \label{section:discussion}
In this section, we discuss the interpretations of our results and their possible extensions.

\subsection{Implicit Regularization} \label{subsection:implicit_regularization}
Intuitively, given a (possibly infinite rank) PDS kernel $K$, one decomposes the kernel matrix into: $\mathbf{K}=\mathbf{K}_{\leq l}+\mathbf{K}_{>l}$ where the low-rank part $\mathbf{K}_{\leq l}$ fits the low-complexity target function while the high-rank part $\mathbf{K}_{>l}\approx (\sum_{k>l}\lambda_k)\mathbf{I}_N$ serves as the implicit regularization. Hence the (normalized) effective rank $\rho_l\eqdef\frac{\sum_{k>l}\lambda_l}{N\lambda_{l+1}}$ measures the relative strength of the implicit regularization.
With exponential eigen-decay, the effective rank $\rho_l=\Theta(N^{-1})\ll O(1)$ is negligible, hence one can expect the catastrophic overfitting as the implicit regularization is not strong enough to stop the interpolant using high-frequency eigenfunctions to fit the noise. With polynomial decay, the effective rank $\rho_l= \Theta(1)$ shows that the interpolant would fit the white noise as if it is the target function, hence overfitting is tempered; for even slower decay like linear-poly-logarithmic decay $\lambda_k=\Theta(\frac{1}{k\log^{2} k})$ in \cite{barzilai2023generalization}, the effective rank $\rho_l=\Omega(\log l)$, hence the high-frequency part is heavily regularized and benign overfitting would occur.

\subsection{Sub-Gaussian Design} \label{subsection:sub_gaussian_design}
We emphasize that the sub-Gaussian design assumption (Assumption \ref{assumption:sub_gaussian_design}) represents a significantly weaker assumption compared to the Gaussian design assumption (Assumption \ref{assumption:gaussian_design}), which enhances the theoretical significance of our paper over previous literature in several ways:
\begin{enumerate}
    \item We only necessitate the independence of sub-Gaussian variables, not their identical distribution, unlike previous works such as \cite{bordelon2020spectrum, cui2021generalization, loureiro2021learning}, which relied on the Replica Method and could not circumvent the Gaussian design assumption.
    \item In general, sub-Gaussian vectors lack the rotational invariance property of Gaussian vectors, which was crucial in the analyses of \cite{simon2021eigenlearning, mallinar2022benign}.
    \item Sub-Gaussian variables are not required to be continuous, unlike Gaussian vectors, where the continuity of kernels (and consequently the feature vectors) is often assumed in KRR literature \cite{zhang2023optimality, li2023kernel, li2023on, haas2023mind}.
\end{enumerate}

\subsection{Beyond Independent Features} \label{subsection:independent_features}
Comparing independent and dependent features, Theorem \ref{theorem:main:3} offers insights into the lower bound of the smallest singular value for polynomial eigen-decay and dependent sub-Gaussian features. Consequently, this revelation underscores the qualitative difference in generalization behavior between independent and dependent features: if the sub-Gaussian features are dependent, overfitting can escalate to catastrophic levels (see Figure \ref{fig:overfitting}); conversely, independent sub-Gaussian features imply tempered overfitting.
As a result, our paper rigorously demonstrates that the theories presented in \cite{bordelon2020spectrum, cui2021generalization, simon2021eigenlearning} fail to accommodate the possibility of catastrophic overfitting with polynomial eigen-decay. Therefore, our paper provides insights that are currently absent in the field. It underscores the pivotal role of independence versus dependent features and prompts further inquiry into identifying additional properties of features that influence tempered or catastrophic overfitting.
 
\subsection{Limitations} \label{subsection:limitations}
Our method currently operates solely under the sub-Gaussian design assumption (Assumption \ref{assumption:sub_gaussian_design}), presuming the kernel rank is finite, and the features are independent of each other. We acknowledge that this remains distant from the realistic kernel setting. 

However, we have the following justifications for the limitations in this paper.
\paragraph{Finite Rank Features}
To justify whether the finite rank approximation is sufficient to investigate overfitting behaviour, we present the following convergence result of the variance term $\V$. Fix a (infinite-rank) kernel with Mercer decomposition $K=\sum_{k=1}^\infty\lambda_k\psi_k(\cdot)\psi_k(\cdot)$ and a sample of size $N$. For each integer $M\in\N$, define the truncated kernel $K^{(M)}=\sum_{k=1}^M\lambda_k\psi(\cdot)\psi(\cdot)$. Let $\V$ and $\V(M)$ be the variance terms corresponding to the kernels $K$ and $K^{(M)}$ respectively. Then there exists an integer $M_0$ such that:
\begin{equation*}
    |\V-\V(M)|
    \leq
    3\V(M) + \frac{\sigma^2}{N}
\end{equation*}
whenever $M>M_0$. In particular, the variance $\V(M)$ of a finite rank kernel $K^{(M)}$ has the same overfitting behaviour as the original one $\V$. 
See Proposition \ref{proposition:finite_rank} in the appendix for more details.
\paragraph{Concurrent Work}
We are aware of the concurrent work \cite{barzilai2023generalization}, which addresses the same problem and offers statements valid for a broader class of features, particularly including kernel features.
However, our work serves as a complement to theirs. For instance, our analysis can elucidate overfitting in cases of exponential decay (where their bounds may be vacuous).

\subsection{Future Research} \label{subsection:future}
There are several obvious possibilities to extend the results of this paper:

\begin{enumerate} 
    \item What causes NTK to exhibit catastrophic overfitting while Laplacian exhibits tempered overfitting? There is more than just the eigen-spectrum that affects the overfitting behaviour, which is worth further investigation.
    \item The work \cite{simon2021eigenlearning} suggested that the distribution of the kernel features with realistic data is similar to Gaussian. Does this suggest that the data distribution in a realistic dataset leads to independent eigen-functions? As we have seen in our paper, feature independence plays an important role in overfitting behaviour. This might help us to understand more about benign overfitting reported in \cite{zhang2017understanding}. 
    \item Controlling the condition number of the kernel matrix in Theorem \ref{theorem:main:1} can be of independent interest, for instance when studying the convergence properties of gradient-based methods.
\end{enumerate}

\bibliographystyle{unsrtnat} 
\bibliography{paper.bbl}
\newpage
\appendix
\onecolumn

\vbox{
  {\hrule height 2pt \vskip 0.15in \vskip -\parskip}
  \centering
  {\LARGE\bf Appendix\par}
  {\vskip 0.2in \vskip -\parskip \hrule height 0.5pt \vskip 0.09in}
}

With abuse of notations, the constants $c,c_1,c_2,...$ with small letter $c$ may change from line to line.

Denote $\|\mathbf{v}\|_\mathbf{M}\eqdef \sqrt{\mathbf{v}^\top\mathbf{M}\mathbf{v}}$ for any vector $\mathbf{v}$ and matrix $\mathbf{M}$ with appropriate dimension.

For an $M\times N$ matrix $\M$, denote $\M_{\leq l}\in\R^{l\times N}$ its submatrix containing the first $l$ columns; for $M\times M$ square matrix $\mathbf{S}$, denote $\M_{\leq l}\in\R^{l\times l}$ its submatrix containing the first $l$ columns and rows. Matrices with subscripts $\cdot_{l_1:l_2}$ or $\cdot_{>l}$ are defined similarly.

\section{Proof} \label{section:sub_gaussianity}
In this section, we will prove the Theorem \ref{theorem:main:1} on the condition number of $\mathbf{K}$ under Assumption \ref{assumption:sub_gaussian_design}. 

Let us first restate the sub-Gaussian design assumption (Assumption \ref{assumption:sub_gaussian_design}):
\begin{assumption}[Sub-Gaussian design]
    Let $M\in\N$. For every $k\in\N$, the random variable $\psi_k(x)$ is replaced by an independent sub-Gaussian variable in $\R^M$ with uniformly bounded sub-Gaussian norm.
\end{assumption}

Consider the regressor 
\begin{equation*}
    \hat{f}(x)
    =
    \K_x(\K + N \lambda \I)^{-1}\y
    =
    \ps(x)^\top \Ps (\Ps\La\Ps^\top + N\lambda\I)^{-1}\y
\end{equation*}
Effectively, the Sub-Gaussian Design Assumption replace the vector $\ps(x)$ and the columns $\Ps_i$ of the matrix $\Ps\in\R^{M\times N}$ by sub-Gaussian vectors with independent entries. Note that, by setup, those vectors $\ps(x)$ and $\Ps_i$ are i.i.d. to each other.

\subsection{Condition number}
The control on the largest singular value $s_{\max}(\K)$ of the kernel matrix directly follows from the literature:
\begin{lemma}[bound on largest singular value, Theorem 9 in \cite{koltchinskii2017concentration}, Theorem 1 in \cite{zhivotovskiy2024dimension}] \label{lemma:s_max:ub:sub_gaussian}
    Suppose Assumption \ref{assumption:sub_gaussian_design} holds, that is, there exists some constant $\kappa>1$ such that 
    \begin{equation*}
        \norm{\langle \v, \La^{1/2}\ps \rangle}{\psi_2}
        \leq
        \kappa \sqrt{\v^\top \La \v}
    \end{equation*}
    for all $\v\in\R^n$, where $\ps$ is the random sub-Gaussian vector with the columns $\Ps_i$ in $\Ps$ as its realization, and $\norm{\cdot}{\psi_2}$ denote the sub-Gaussian norm.
    Then with probability at least $1-e^{-t}$, it holds that
    \begin{equation*}
        \opnorm{\frac{1}{N}\bm{\Lambda}^{1/2}\bm{\Psi}\bm{\Psi}^\top\bm{\Lambda}^{1/2}-\bm{\Lambda}}
        \leq
        20 \kappa^2 \opnorm{\bm{\Lambda}}\sqrt{ 4\rho_0 + \frac{t}{N}}
    \end{equation*}
    whenever $N\geq 4N\rho_0+t$, and $\rho_0\eqdef\frac{\tr[\bm{\Lambda}]}{N\opnorm{\bm{\Lambda}}}$ is the normalized effective rank of $\bm{\Lambda}$.\\
    In particular, if $\rho_0\leq \frac{1}{80(40\kappa^2)^2}$, then with probability at least $1-e^{-\frac{N}{2(40\kappa^2)^2}}$, it holds that
    \begin{equation*}
        \half N\opnorm{\bm{\Lambda}}
        \leq
        s_{\max}(\Ps^\top\La\Ps)
        =
        s_{\max}(\bm{\Lambda}^{1/2}\bm{\Psi}\bm{\Psi}^\top\bm{\Lambda}^{1/2})
        \leq
        \frac{3}{2} N\opnorm{\bm{\Lambda}}.
    \end{equation*}
\end{lemma}
\begin{remark}[Sub-Gaussian Condition] \label{remark:centering}
    In \cite{koltchinskii2017concentration, zhivotovskiy2024dimension}, the random vector $\psi_k$ are required to be centered. However, as mentioned Remark 5.18 in \cite{vershynin2010introduction}, centering of a sub-Gaussian random variable $X$ does not change the sub-Gaussian constant by more than 2:
    \begin{equation*}
        \norm{X - \Expect{}{X}}{\psi_2} \leq 2\norm{X}{\psi_2}.
    \end{equation*}
    Hence, by possibly changing the constant $\kappa$, we drop the requirement of centered random variable in the statement.
\end{remark}

\begin{lemma}[Lower bound of smallest singular value for polynomial spectrum] \label{lemma:s_min:lb:poly:sub_gaussian}
    Suppose Assumption \ref{assumption:interpolation} holds, $\lambda_k=\bigtheta{k}{k^{-1-a}}$ for some constant $a>0$, and each column $\Ps_i=(\ps_{ki})_{k=1}^M\in\R^M$ is i.i.d. sub-Gaussian isotropic random vector, whose entries $\psi_{ki}$ are not necessarily independent. Then there exists constants $c_1,c_2>0$ such that, with a probability of at least $1-2e^{-c_1N}$:
\begin{equation} 
    s_{\min}(\mathbf{K})
    \geq  
     c_2\min_i\{P_i^2\}\cdot N \lambda_N,
\end{equation}
where $P_i\eqdef\sqrt{\frac{\sum_{i=N+1}^M\psi_{ki}^2}{M-N}}$ is a random variable depending on the inputs $x_i$'s.
Furthermore, if Assumption \ref{assumption:sub_gaussian_design} holds, then with probability at least $1-2Ne^{-c_1N}$, it holds that
\begin{equation} 
    s_{\min}(\mathbf{K})
    \geq  
     \frac{c_2}{2} N \lambda_N.
\end{equation}
\end{lemma}
\begin{proof}
    By Assumption \ref{assumption:interpolation}, the feature dimension $M\geq \theta N$.  
    First, by Weyl's theorem (Corollary 4.3.15 in \cite{horn2012matrix}): $s_{\min}(\M_1)+s_{\min}(\M_2)\leq s_{\min}(\M_1+\M_2)$ for any any symmetric matrix $\M_1,\M_2\in\R^{N\times N}$, we have
    \begin{equation*}
        s_{\min}(\mathbf{K}) 
        = 
        s_{\min}(\bm{\Lambda}^{1/2}\bm{\Psi}^\top)^2 
        \geq 
        s_{\min}(\bm{\Lambda}_{N:\theta N}^{1/2}\bm{\Psi}_{N:\theta N}^\top)^2 
        \geq 
        \lambda_{N} \cdot \frac{\lambda_{\theta N}}{\lambda_{N}} s_{\min}(\bm{\Psi}_{N:\theta N})^2
    \end{equation*}
    where we write $\bm{\Lambda}^{1/2}\bm{\Psi}^\top=\bm{\Lambda}_{N:\theta N}^{1/2}\bm{\Psi}_{N:\theta N}^\top+(\bm{\Lambda}^{1/2}\bm{\Psi}^\top-\bm{\Lambda}_{N:\theta N}^{1/2}\bm{\Psi}_{N:\theta N}^\top)$ and $\cdot_{N:\theta N}$ denote the submatrix with columns ranging from $N$ to $\theta N$. Hence with abuse of notation, we replace $M$ by $\theta N$ in the following argument. 
    Let $\mathbf{R}_i^\top\in\R^{M-N}=\R^{(\theta-1)N}$ be the $i$-th row of $\bm{\Psi}_{>N}$, and $\hat{\mathbf{R}}_i \eqdef \frac{\sqrt{(M-N)}}{\eunorm{\mathbf{R}_i}} \mathbf{R}_i$. Note that $\Expect{}{\eunorm{\mathbf{R}_i}^2}=(\theta-1)N,\ \forall i=1,...,N$.
    Let $\hat{\bm{\Psi}}_{>N} \eqdef (\hat{\mathbf{R}}_i)_{i=1}^{N} \in\R^{N\times ((\theta-1)N)}$. Now the matrix $\hat{\bm{\Psi}}_{>N}^\top$ is an $((\theta-1)N)\times N$ matrix whose columns $\hat{\mathbf{R}}_i$ are independent sub-Gaussian isotropic random matrix with norm $\eunorm{\hat{\mathbf{R}}_i}=\sqrt{(\theta-1)N}$. Hence, by Theorem \ref{theorem:vershynin:3}, there exists constants $C_8,C_9>0$ (depending only on the sub-Gaussian norm of $\psi_k$) such that, for any $t>0$, with probability at least $1-2e^{-C_8t^2}$, the inequality holds:
    \begin{equation*}
        s_{\min}(\hat{\bm{\Psi}}_{>N}) \geq \sqrt{(\theta-1)N} - C_9 \sqrt{N} - t. 
    \end{equation*}
    Set $t=\sqrt{N}$ and $\theta>(C_9+2)^2+1$, the inequality holds:
    \begin{equation*}
        s_{\min}(\hat{\bm{\Psi}}_{>N}) \geq \sqrt{\theta N-N}- C_9 \sqrt{N} - \sqrt{N} \geq \sqrt{N},
    \end{equation*}
    with probability at least $1-2e^{-c_8N}$.
    Notice that $\bm{\Psi}_{>N}=  \hat{\bm{\Psi}}_{>N}\mathbf{P}$ where $P_i\eqdef \frac{\eunorm{\mathbf{R}_i}}{\sqrt{(\theta-1)N}},\ \forall i=1,...,N$ and $\mathbf{P}\eqdef \diag\{P_i\}_{i=1}^n\in \R^{N\times N}$ is a random matrix with $\Expect{}{\mathbf{P}^2}=\mathbf{I}_N$.  Hence, with high probability,
    \begin{equation*}
        s_{\min}(\bm{\Psi}_{>N})^2 
        \geq  
        s_{\min}(\hat{\bm{\Psi}}_{>N})^2 s_{\min}(\mathbf{P})^2
        =
        \min_i\{P_i^2\} s_{\min}(\hat{\bm{\Psi}}_{>N})^2
        \gtrsim
        \min_i\{P_i^2\}N.
    \end{equation*}
    Since $\frac{\lambda_{M}}{\lambda_{N}} \asymp \frac{M^{-1-a}}{N^{-1-a}} = \frac{(\theta N)^{-1-a}}{N^{-1-a}} = \theta^{-1-a} \asymp 1$, thus $s_{\min}(\mathbf{K})\gtrsim \min_i\{P_i^2\} \cdot N\lambda_N $.
    If, furthermore, Assumption \ref{assumption:sub_gaussian_design} holds, then each row vector $\mathbf{R}_i$ has independent sub-Gaussian entries, that is, write $\mathbf{R}_i^\top = \left(z_i^{(k)}\right)_{k=N+1}^M$ where each $z_i^{(k)}$ is an independent random variable with sub-Gaussian norm $\leq G$. Then by remark \ref{remark:centering}, $P_i^2=\frac{1}{(\theta-1)N}\sum_{k=N+1}^M (z_i^{(k)})^2$ is the average of some independent sub-exponential variables with sub-exponential norms $\leq G^2$. Hence by Lemma \ref{lemma:sub_exponential_deviation}, it holds that
    \begin{equation*}
        \Prob{\left| P_i^2 -1 \right|\geq\delta}
        \leq 2 \Exp{-C_5\min\left\{\frac{\delta^2}{G^4},\frac{\delta}{G^2}\right\}((\theta-1)N)}
    \end{equation*}
    for some absolute constant $C_5>0$. Write $c_1=-C_5\min\left\{\frac{1}{4G^4},\frac{1}{2G^2}\right\}(\theta-1)$. Then with a probability at least $1-2Ne^{-c_1N}$, it holds that
    \begin{equation*}
        P_i^2 \geq \half,\ \forall i=1,...,N. 
    \end{equation*}
\end{proof}
\begin{remark}[Dependence of features]
    We can see that the effect of the dependence of features is encrypted in the term $\min_i\{P_i^2\}$ in Lemma \ref{lemma:s_min:lb:poly:sub_gaussian}. Indeed, the smallest singular value of the kernel matrix with dependent features vanishes (see Figure \ref{fig:co_sine}) while that with independent features remains in the same magnitude of the theoretical lower bound (see Figure \ref{fig:sub_gaussian}). 
\end{remark}

Hence we can bound the condition number for polynomial eigen-decay:
\begin{lemma}[Condition number for polynomial eigen-decay] \label{lemma:condition_number:poly}
     Suppose $M,N\in\N$ such that $M\geq \eta N$ for some constant $\eta>1$ large enough. Let $\Psi\in\R^{M\times N}$ be a matrix with i.i.d. isotropic random vectors $\Psi_i$'s with (possibly dependent) sub-Gaussian entries as columns. Let $\La=\diag(\lambda_k)_{k=1}^M\in\R^{M\times M}$ be a diagonal matrix with $\lambda_k$'s decaying polynomially. Then there exists some constants $c_1,c_2>0$, such that for $N$ large enough, with probability $1-2Ne^{-c_1N}$, we have
    \begin{equation*}
        \frac{s_{\max}(\K)}{s_{\min}(\K)} \leq  c_2\frac{\lambda_1}{\lambda_N},
    \end{equation*}
\end{lemma}
\begin{proof}
    Simply combine Lemmata \ref{lemma:s_max:ub:sub_gaussian} and \ref{lemma:s_min:lb:poly:sub_gaussian}. By possibly choosing a larger constant $c_1$, the claim holds with probability $1-2Ne^{-c_1N}$, for $N$ large enough. 
\end{proof}

For exponential eigen-decay, we use another argument:
\begin{lemma}[Condition number for exponential eigen-decay] \label{lemma:condition_number:exp}
Suppose $M=\eta N$ for some integer $\eta>1$ large enough. Let $\K=\Ps^\top\La\Ps\in\R^{N\times N}$ be a random matrix where $\Ps\in\R^{M\times N}$ is a Gaussian random matrix, and $\La=\diag(\lambda_k)_{k=1}^\infty$ is a diagonal matrix with $\lambda_k=\bigtheta{k}{e^{-ak}}$ for some constant $a>0$. Then there exist constants $c_1,c_2>0$ such that for $N$ large enough, with probability at least $1-\delta-2/N$,
\begin{equation*}
    c_1 \frac{\lambda_1}{\lambda_N} N 
    \leq
    \frac{s_{\max}(\K)}{s_{\min}(\K)}
    \leq
    \frac{c_2}{\delta^2} \frac{\lambda_1}{\lambda_N}N.
\end{equation*}
\end{lemma}
\begin{proof}
    By Lemma \ref{lemma:s_max:ub:sub_gaussian}, for $N$ large enough, with probability at least $1-\frac{1}{N}$, we have
    \begin{equation} \label{line:s_max}
        \half N\lambda_1 \leq s_{\max}(\K) \leq \frac{3}{2} N\lambda_1.
    \end{equation}
    It remains to show that $s_{\min}$ is bounded above and below at the magnitude of $\lambda_N$.

    For the upper bound, 
    fix the first $N-1$ vectors $\psi_1,..,\psi_{N-1}$ and pick $v_0\in\mathbb{S}^{N-1}$ orthogonal to them. Then 
\begin{equation*} 
    s_{\min}(\mathbf{K})
    = \inf_{v\in\mathbb{S}^{N-1}} \sum_{k=1}^M \lambda_k (\bm{\psi}_k^\top v)^2  
    \leq \sum_{k=1}^M \lambda_k (\bm{\psi}_k^\top v_0)^2 
    \leq \sum_{k=N}^M \lambda_k (\bm{\psi}_k^\top v_0)^2.
\end{equation*}
Since the Gaussian is rotational invariant, we have $(\bm{\psi}_k^\top v_0)^2\sim \chi^2(1)$.
By Lemma \ref{lemma:sub_exponential_deviation}, hence we have
\begin{equation*}
    \Prob{\left|(\bm{\psi}_k^\top v_0)^2-1\right|\geq t} \leq 2e^{-t^2/8}.
\end{equation*}
Set $t=\sqrt{8\log\frac{2(\theta-1)N}{\delta}}$ and By the union bound, we have
\begin{equation*}
    \Prob{\left|(\bm{\psi}_k^\top v_0)^2-1\right|\leq t : N \leq k\leq M}
    \geq 1 - \sum_{k=N}^M \frac{\delta}{(\theta-1)N}
    \geq 1- \delta.
\end{equation*}
Thus with probability of at least $1-\delta$, we have
\begin{equation}\label{line:s_min:ub}
    s_{\min}(\mathbf{K})
    \leq \sum_{k=N}^M \lambda_k (1+t)
    = \left(1+\sqrt{8\log\frac{2(\theta-1)N}{\delta}}\right) \sum_{k=N}^M\lambda_k
\end{equation}

Since $\lambda_k=\bigtheta{k}{e^{-ak}}$, there exists some constant $c>0$ such that
\begin{equation*} 
   \sum_{k=N}^M\lambda_k \leq c \lambda_N;
\end{equation*}
By setting $\delta=\frac{1}{N}$, the factor $\left(1+\sqrt{8\log\frac{2(\theta-1)N}{\delta}}\right)$ becomes constant in line (\ref{line:s_min:ub}) and hence with probability at least $1-\frac{1}{N}$,
\begin{equation}\label{line:s_min:ub:exp}
    s_{\min}(\K) 
    \leq
    c\lambda_N
\end{equation}
for some constant $c>0$.

For the lower bound,
    let $\mathbf{K}_N=\sum_{k=1}^N \lambda_k\bm{\psi}_k\bm{\psi}_k^\top \prec \mathbf{K}$. Let $\bm{\Lambda}_N=\diag(\lambda_k)_{k=1}^N\in\R^{N\times N}$ and $\bm{\Psi}_N=(\bm{\psi}_k)_{k=1}^N\in\R^{N\times N}$ and set $\mathbf{M}=\bm{\Lambda}_N^{1/2}\bm{\Psi}_N$ which is invertible almost surely. Note that $\mathbf{K}_N=\mathbf{M}^\top\mathbf{M}$.
    Let $\mathbf{R}_1,...,\mathbf{R}_n$ be the rows of $\mathbf{M}$ and let $\mathbf{C}_1,...,\mathbf{C}_n$ be the columns of $\mathbf{M}^{-1}$. For each $1\leq i\leq n$, let $\mathbf{N}_i$ be a unit normal vector orthogonal to the subspace spanned by all rows $\mathbf{R}_1,...,\mathbf{R}_n$ except $\mathbf{R}_i$. 

By Lemma \ref{lemma:anti_concentration_gaussian}, for any $k\leq N$ and $t\in(0,\infty)$,
\begin{align*}
    \Prob{\frac{\lambda_N}{\lambda_k}(\bm{\psi}_k^\top\mathbf{N}_k)^{-2}\geq t^{-1}e^{-\frac{a}{2}(N-k)}}
    &= \Prob{|\bm{\psi}_k^\top\mathbf{N}_k| \leq \sqrt{\frac{\lambda_N}{\lambda_k}te^{\frac{a}{2}(N-k)}}}\\
    &\leq \frac{2}{\sqrt{2\pi}} \cdot \sqrt{\frac{\lambda_N}{\lambda_k}te^{\frac{a}{2}(N-k)}}\\
    &\leq \frac{2}{\sqrt{2\pi}} \cdot \sqrt{\frac{\overline{r}e^{-aN}}{\underline{r}e^{-ak}}}e^{\frac{a}{4}(N-k)} \sqrt{t}\\
    &= \frac{2}{\sqrt{2\pi}} \cdot \sqrt{\frac{\overline{r}}{\underline{r}}}e^{-\frac{a}{4}(N-k)} \sqrt{t}.
\end{align*}

for all $k=1,...,N$. By the union bound, we have
\begin{align*}
    \Prob{\underbrace{\frac{\lambda_N}{\lambda_k}(\bm{\psi}_k^\top\mathbf{N}_k)^{-2}\leq t^{-1}e^{-\frac{a}{2}(N-k)}:\ \forall k=1,...,N}_{E}}
    &\geq 1 - \sum_{k=1}^N \frac{2}{\sqrt{2\pi}} \cdot \sqrt{\frac{\overline{r}}{\underline{r}}}e^{-\frac{a}{4}(N-k)} \sqrt{t}\\
    &\geq 1- \frac{2}{\sqrt{2\pi}} \cdot \sqrt{\frac{\overline{r}}{\underline{r}}}(1-e^{-a/4})^{-1}\sqrt{t} .
\end{align*}
When the event $E$ happens, we have 
\begin{equation*}
    \sum_{k=1}^N \frac{\lambda_N}{\lambda_k}(\bm{\psi}_k^\top\mathbf{N}_k)^{-2}
    \leq \sum_{k=1}^N t^{-1}e^{-\frac{a}{2}(N-k)} 
    \leq (1-e^{-a/2})^{-1}t^{-1},
\end{equation*}
by Lemma \ref{lemma:s_min:lb}, with probability at least $1-\frac{2}{\sqrt{2\pi}} \cdot \sqrt{\frac{\overline{r}}{\underline{r}}}(1-e^{-a/4})^{-1}\sqrt{t}$, we have
\begin{equation*}
    s_{\min}(\mathbf{K}) \geq \lambda_N(1-e^{-a/2}) t
\end{equation*}
for any $t>0$. Set $\delta=\frac{2}{\sqrt{2\pi}} \cdot \sqrt{\frac{\overline{r}}{\underline{r}}}(1-e^{-a/4})^{-1}\sqrt{t}$ and we have: with probability at least $1-\delta$, 
\begin{equation} \label{line:s_min:lb:exp}
    s_{\min}(\K) \geq c\delta^2\lambda_N
\end{equation}
for some constant $c>0$.

Combining Eq. (\ref{line:s_max}), (\ref{line:s_min:ub:exp}) and (\ref{line:s_min:lb:exp}), we obtain the claim.
\end{proof}

\subsection{Test Error}

In this paper, we use the bias-variance decomposition to analyse the test error $\mathcal{R}$, which is common in literature: \cite{bartlett2020benign, tsigler2023benign, bach2023high, li2023kernel, li2023on}.

With abuse of notation, we write $f^*(\mathbf{\mathbf{X}})\in\R^N$ to be the evaluation of $f^*$ on the training set $\mathbf{X}=(x_i)_{i=1}^N$.

\begin{definition}[Bias-Variance Decomposition of test error] \label{definition:test_error}
    Given the test error $\mathcal{R}\eqdef \Expect{x,\epsilon}{(f^{\star}(x)-\hat{f}(x))^2}$ be the test error. Define the \textit{bias}
    \begin{equation*}
        \mathcal{B}\eqdef \Expect{x}{(f^{\star}(x)-\mathbf{K}_{x}^\top\mathbf{K}[f^{\star}(\mathbf{X})])^2},
    \end{equation*}
    which measures how accurately the KRR approximates the true target function $f^{\star}$.  
    The \textit{variance}, defined as the difference
    \begin{equation*}
        \mathcal{V}=\mathcal{R}-\mathcal{B}
    \end{equation*} 
    quantifies the impact which overfitting to noise has on the test error.
\end{definition}

Together with Theorem \ref{theorem:tsigler} from \cite{tsigler2023benign}, we obtain the first statement of Theorem \ref{theorem:main:1}:
\begin{theorem} \label{theorem:test_error:poly}
    Suppose the interpolation assumption (Assumption \ref{assumption:interpolation}) and sub-Gaussian design assumption (Assumption \ref{assumption:sub_gaussian_design}) hold. Then there exists some constants $c_1,c_2,c_3,c_4,c_5,c_6$ such that with probability at least $1-c_1e^{-N/c_1}-2Ne^{-c_2N}$, we have
    \begin{align*}
        \mathcal{B} &\leq c_3\|\theta_{>\lfloor N/c_1 \rfloor}^*\|_{\bm{\Lambda}_{>\lfloor N/c_1 \rfloor}}^2+
        c_4\|\bm{\theta}_{\leq \lfloor N/c_1 \rfloor}^*\|^2\lambda_{\lfloor N/c_1 \rfloor};\\
        \mathcal{V} &\leq c_5 + \frac{c_6}{N}.
    \end{align*}
    where the target function $f^*$ is given by the inner product $\langle\th^*,\La^{1/2}\cdot\rangle$, where $\th^*\in\R^M$ is a deterministic vector with $\norm{\th^*}{\La}<\infty$. In particular, there exists a constant $C>0$ independent to $N$ such that
    \begin{equation*}
        \mathcal{R}
        =
        \mathcal{B} + \mathcal{V}
        \leq C.
    \end{equation*}
\end{theorem}
\begin{proof}
    By Theorem \ref{theorem:tsigler}, 
    there exists a constant $c>0$, such that for any $l\leq N/c$, with probability of at least $1-ce^{-N/c}$, the bias and the variance is bounded by:
    \begin{align*}
        \begin{split}
            \mathcal{B}/c &\leq \|\bm{\theta}_{>l}^*\|_{\bm{\Lambda}_{>l}}^2  \left(1+\frac{s_1(\K_l^{-1})^2}{s_N(\K_l^{-1})^2}+N\lambda_{l+1}s_1(\K_l^{-1})\right)\\
            &\quad+ \|\bm{\theta}_{\leq l}^*\|_{\bm{\Lambda}_{\leq l}^{-1}}^2\left(\frac{1}{N^2s_N(\K_l^{-1})^2}+\frac{\lambda_{l+1}}{N}\frac{s_1(\K_l^{-1})}{s_N(\K_l^{-1})^2}\right)
        \end{split} \\
        \mathcal{V}/c &\leq \frac{s_1(\K_l^{-1})^2}{s_N(\K_l^{-1})^2}\frac{l}{N} + Ns_1(\K_l^{-1})^2\sum_{k>l}\lambda_k^2,
    \end{align*}
    where $\K_{l}\eqdef\Ps_{>l}^\top\La_{>l}\Ps_{>l}$,   $\bm{\theta}^*=\bm{\theta}_{\leq l}^*\dirsum\bm{\theta}_{>l}^*$ is the splitting of the target function coefficient.
    Take $l=\lfloor N/c \rfloor$. Since $\lambda=0$, so
    $s_N(\K_l^{-1})=s_1(\K_l)^{-1}$ and $s_1(\K_l^{-1})=s_N(\K_l)^{-1}$. Hence $\frac{s_1(\K_l^{-1})^2}{s_N(\K_l^{-1})^2}=\frac{s_{\max}(\K_l)^2}{s_{\min}(\K_l)^2}$. Since $\K_l$ is just another kernel matrix with rank $(M-l)$, by modifying Lemma \ref{lemma:condition_number:poly} w.r.t. the right-shifted polynomial decay, with probability $1-2Ne^{-c_2N}$, we have
    \begin{equation*}
        \frac{s_{\max}(\K_l)}{s_{\min}(\K_l)}
        \lesssim \frac{\lambda_{l+1}}{\lambda_{l+N}} 
        \lesssim \frac{l^{-a}}{(l+N)^{-a}} 
        = (1+N/l)^a
        \leq (1+N/(N/c))^a
        = (1+c)^a,
    \end{equation*}
    and
    \begin{align*}
        N\lambda_{l+1}s_1(\K_l^{-1})
        = N\lambda_{l+1}s_N(\K_l)^{-1}
        \lesssim N \frac{\lambda_{l+1}}{N\lambda_{l+N}} 
        =\frac{\lambda_{l+1}}{\lambda_{l+N}},
    \end{align*}
    then we can bound the bias term using Theorem \ref{theorem:tsigler}:
    \begin{align*}
        \mathcal{B}/c 
        &\leq c_1\|\theta_{>l}^*\|_{\bm{\Lambda}_{>l}}^2
        + \|\bm{\theta}_{\leq l}^*\|_{\bm{\Lambda}_{\leq l}^{-1}}^2
        \left(\frac{c_2N^2\lambda_{l+1}^2}{N^2}+\frac{\lambda_{l+1}}{N}\frac{c_3N^2\lambda_{l+1}^2}{N\lambda_{l+N}}\right)\\
        &\leq c_1\|\theta_{>l}^*\|_{\bm{\Lambda}_{>l}}^2+
        c_2\|\bm{\theta}_{\leq l}^*\|_{\bm{\Lambda}_{\leq l}^{-1}}^2
        \lambda_l^2\\
        &\leq c_1\|\theta_{>l}^*\|_{\bm{\Lambda}_{>l}}^2+
        c_2\|\bm{\theta}_{\leq l}^*\|^2\lambda_l.
    \end{align*}
    Similarly, we can write the variance term into:
    \begin{align*}
        \mathcal{V}/c &\leq c_3\frac{l}{N} + \frac{N}{N^2\lambda_{l+N}^2}\sum_{k>l}\lambda_k^2\\
        &\leq c_3\frac{l}{N} + \frac{c_4}{N^2\lambda_{l+N}^2}\int_l^\infty t^{-2a}dt\\
        &= c_3\frac{l}{N} + \frac{c_4}{N^2\lambda_{l+N}^2}l^{-2a+1}\\
        &= c_3\frac{l}{N} + \frac{c_4}{N^2(l+N)^{-2a}}l^{-2a+1}\\
        &= c_3 + \frac{c_4}{N},
    \end{align*}
    since $l=\lfloor N/c \rfloor$.
\end{proof}
\begin{corollary}[tempered overfitting] \label{corollary:tempered}
    There exists a constants $C\in(0,1)$ independent to $N$ such that, with the same probability as in Theorem \ref{theorem:test_error:poly}, we have
    \begin{equation*}
        C\leq \mathcal{R}  \leq C^{-1}.
    \end{equation*}
\end{corollary}
\begin{proof}
    It is a direct consequence of Theorems \ref{theorem:tsigler} and \ref{theorem:tsigler:2} about the non-asymptotic upper bound of the variance $\mathcal{V}$ and its matching lower bound. 
    In more details, we compute the (normalized) effective rank:
    \begin{equation*}
        \rho_l \eqdef \frac{1}{N\lambda_{l+1}}\sum_{k=l+1}^M\lambda_k
        \asymp \frac{N\lambda_{l+1}}{N\lambda_{l+1}}
        \asymp 1
    \end{equation*}
    for all $l=1,...,M-1$.
    Hence the condition (i) in Theorem \ref{theorem:tsigler:2} would hold for some $l=N/c_1$ where $c_1>1$. Then we apply Theorem \ref{theorem:tsigler:2} for polynomial decay:
    there exists constants $C,C'$, with a probability at least $1-Ce^{-N/C}$, we have 
    \begin{align*}
        \mathcal{V} 
        \geq C' \left( \frac{l}{N} + \frac{N\sum_{k>l}\lambda_k^2}{\left(\sum_{k>l}\lambda_k\right)^2} \right)
        = \Omega \left( \frac{l}{N} + \frac{N\int_l^\infty t^{-2a}dt}{\left(\int_l^\infty t^{-2a}\right)^2} \right)
        = \Omega(1).
    \end{align*}
    Hence, combining the result with Theorem \ref{theorem:test_error:poly}, there exists constants $c_1,c_2>0$ independent to $N$ such that $c_1\leq \mathcal{R} \leq c_2$ with the probability stated in Theorem \ref{theorem:test_error:poly}. Take $C=\max\{(c_1^{-1}+1)^{-1},(c_2+1)^{-1}\}$ to conclude the claim.
\end{proof}

\begin{theorem}[catastrophic overfitting] \label{theorem:catastrophic}
    Suppose interpolation assumption (Assumption \ref{assumption:interpolation}) and sub-Gaussian design assumption (Assumption \ref{assumption:gaussian_design}) hold. Then there exists some constants $c_1,c_2>0$ independent to $N$ such that with probability at least $1-c_1e^{-N/c_1}$, we have
    \begin{equation*}
        \mathcal{R} \geq c_2 N.
    \end{equation*}
\end{theorem}
\begin{proof}
    We argue with Theorem \ref{theorem:tsigler:2} again.
    We compute the (normalized) effective rank:
    \begin{equation*}
        \rho_l \eqdef \frac{1}{N\lambda_{l+1}}\sum_{k=l+1}^M\lambda_k
        \asymp \frac{\lambda_{l+1}}{N\lambda_{l+1}}
        \asymp \frac{1}{N}
    \end{equation*}
    for all $l=1,...,M-1$.
    Hence the condition (i) or (ii) in Theorem \ref{theorem:tsigler:2} would hold for some $l<N$. Then we apply Theorem \ref{theorem:tsigler:2} for exponential decay:
    there exists a constant $C,C'$, with a probability at least $1-Ce^{-N/C}$, we have 
    \begin{align*}
        \mathcal{R}
        \geq 
        \mathcal{V} 
        \geq C' \left( \frac{l}{N} + \frac{N\sum_{k>l}\lambda_k^2}{\left(\sum_{k>l}\lambda_k\right)^2} \right)
        =\Omega\left(\frac{N e^{-2al}}{e^{-2al}}\right) 
        = \Omega(N).
    \end{align*}
    
\end{proof}

It may be of independent interest for the trivial bound of $s_{\min}$.
\begin{lemma}[Trivial bound of the smallest singular value] \label{lemma:s_min:lb:trivial}
    Suppose the entries of the feature vector $\ps\in\R^p$ are i.i.d. draws of a sub-Gaussian variable. (In particular, Assumption \ref{assumption:sub_gaussian_design} holds.) Then there exists constants $c_1,c_2>0$ such that, for any $\epsilon>0$, with probability at least $1-c_1\epsilon-e^{-c_2N}$, it holds that
    \begin{equation*}
        s_{\min}(\K) \geq \frac{\epsilon^2\lambda_N}{N}.
    \end{equation*}
\end{lemma}
\begin{proof}
Observe that
\begin{equation*}
    s_{\min}(\mathbf{K}) 
    \geq 
    s_{\min}(\bm{\Psi}_{\leq N}^\top\bm{\Lambda}_{\leq N}\bm{\Psi}_{\leq N})
    \geq 
    \lambda_N s_{\min}(\bm{\Psi}_{\leq N}^\top\bm{\Psi}_{\leq N})
    = 
    \lambda_N s_{\min}(\bm{\Psi}_{\leq N})^2
    \geq
    \lambda_N \frac{\epsilon^2}{N}
    \geq \frac{\epsilon^2\lambda_N}{N}.
\end{equation*}
where second last inequality holds with probability at least $1-c_1\epsilon-e^{-c_2N}$ by Theorem \ref{theorem:rudelson:2}. 
\end{proof}

\subsection{Finite rank approximation of kernels}

In this subsection, we discuss approximating the overfitting behavior of kernel ridge regression using truncated kernels. This serves as a justification for the finite rank assumption in both Gaussian and sub-Gaussian design assumptions.

Since we focus on overfitting behaviour, where the variance term $\V$ dominates over the bias term $\B$ (see \cite{cui2021generalization,li2023kernel} for details), we show that the variance term from the infinite rank kernel is close to that from its high-rank truncation.

\begin{proposition}[Finite rank kernel] \label{proposition:finite_rank}
    Let $K$ be a PDS kernel with Mercer decomposition 
    \begin{equation*}
        K(x,x')
        =
        \sum_{k=1}^\infty \lambda_k \psi_k(x) \psi_k(x')
    \end{equation*}
    with strictly positive eigenvalues $\lambda_1\geq\lambda_2\geq...$ and corresponding eigenfunction $\psi_k$'s. For any integer $M>N$, define its truncation:
    \begin{equation*}
        K^{(M)}(x,x')
        =
        \sum_{k=1}^M \lambda_k \psi_k(x) \psi_k(x').
    \end{equation*}
    Denote by $\V$ the variance corresponding to the kernel $K$, and by $\V(M)$ that corresponding to $K^{(M)}$.
    Fix a random sample $\{x_i\}_{i=1}^N$ of size $N$. Then there exists an integer $M_0>N$, such that 
    \begin{equation*}
        |\V-\V(M)| \leq  3\V(M)+\frac{\sigma^2}{N}
    \end{equation*}
    whenever $M>M_0$.
\end{proposition}
\begin{proof}
    Consider the variance expression in Lemma \ref{lemma:variance:expression}, we have:
    \begin{equation*}
        \V = \sigma^2 \tr\left[(\Ps^\top\La^2\Ps)(\Ps^\top\La\Ps)^{-2}\right], \quad
        \V(M) = \sigma^2 \tr\left[(\Ps_{\leq M}^\top\La_{\leq M}^2\Ps_{\leq M})(\Ps_{\leq M}^\top\La_{\leq M}\Ps_{\leq M})^{-2}\right].
    \end{equation*}
    To simplify the notation, let 
    \begin{align*}
        \A_1 &= \Ps_{\leq M}^\top\La_{\leq M}\Ps_{\leq M},\quad 
        \De_1 = \Ps_{> M}^\top\La_{> M}\Ps_{> M}\\
        \A_2 &= \Ps_{\leq M}^\top\La_{\leq M}^2\Ps_{\leq M},\quad
        \De_2 = \Ps_{> M}^\top\La_{> M}^2\Ps_{> M}.    
    \end{align*}
    Note that the matrices $\A_1,\A_2,\De_1,\De_2$ depends on $M$ and are PDS a.s.
    Write the singular values as functions of $M$: 
    \begin{align*}
        p_1(M)&=s_{\max}(\De_1),\quad q_1(M)=s_{\min}(\A_1)\\
        p_2(M)&=s_{\max}(\De_2),\quad q_2(M)=s_{\min}(\A_2).
    \end{align*}
    
    By Lemma \ref{lemma:horn}, both $p_1$ and $p_2$ are decreasing functions in $M$, and both $q_1$ and $q_2$ are increasing functions in $M$. Moreover, by the entry-wise convergence of $K^{(M)}(x,x')\to K(x,x')$, we have a.s.:
    \footnote{In more details, we have $0\leq\lim_{M\to\infty}s_{\max}(\De_1)\leq \lim_{M\to\infty}\norm{\De_1}{F}\to0$; by Lemma \ref{lemma:horn} and Weyl's interlacing Theorem, $s_{\min}(\K)\geq\lim_{M\to\infty}s_{\min}(\A_1)\geq \lim_{M\to\infty}(s_{\min}(\K)-s_{\max}(\De_1))\geq s_{\min}(\K)-\lim_{M\to\infty}(s_{\max}(\De_1))\to s_{\min}(\K)$. Argue similarly for $p_2$ and $q_2$. }
    \begin{align*}
        \lim_{M\to\infty}p_1(M)=0&,\quad \lim_{M\to\infty}q_1(M)=Q_1,\\
        \lim_{M\to\infty}p_2(M)=0&,\quad \lim_{M\to\infty}q_2(M)=Q_2.
    \end{align*}
    where $Q_1\eqdef s_{\min}(\Ps^\top\La\Ps) >0$ and $Q_2\eqdef s_{\min}(\Ps^\top\La^2\Ps) >0$ a.s.
    
    Fix an $\epsilon>0$, then there exists an integer $M_0>N$ such that
    \begin{align*}
        p_1(M)&\leq \epsilon,\quad q_1(M)\geq \frac{Q_1}{2},\\
        p_2(M)&\leq \epsilon,\quad q_2(M)\geq \frac{Q_2}{2},
    \end{align*}
    whenever $M>M_0$.
    In such case, we have:
    \begin{align*}
        (\Ps^\top\La\Ps)^{-2}
        &=
        \left((\A_1+\De_1)^{-1}\right)^2\\
        &=
        \left((\A_1^{-1}-\A_1^{-1}\De_1(\A_1+\De_1)^{-1}\right)^2\\
        &=
        \A_1^{-2} \left(\I-\De_1(\A_1+\De_1)^{-1}\right)^2,
    \end{align*}
    where we use the identity of matrix inverse difference: $(\M_1+\M_2)^{-1} = \M_1^{-1} - \M_1^{-1}\M_2(\M_1+\M_2)^{-1} $.
    Then,  
    \begin{align*}
        \opnorm{\left(\I-\De_1(\A_1+\De_1)^{-1}\right)^2}
        &\leq
        \opnorm{\I-\De_1(\A_1+\De_1)^{-1}}^2\\
        &\leq
        \left(1+\opnorm{\De_1(\A_1+\De_1)^{-1}}\right)^2\\
        &\leq
        \left(1+\opnorm{\De_1}\opnorm{(\A_1+\De_1)^{-1}}\right)^2\\
        &\leq 
        \left(1+\opnorm{\De_1}\opnorm{\A_1^{-1}}\right)^2\\
        &= 
        \left(1+s_{\max}(\De_1)s_{\min}(\A_1)^{-1}\right)^2\\
        &\leq
        \left(1+\epsilon\cdot\frac{2}{Q_1}\right)^2.
    \end{align*}
    Hence
    \begin{align*}
        \V/\sigma^2
        &=
        \tr\left[(\Ps^\top\La^2\Ps)(\Ps^\top\La\Ps)^{-2}\right]\\
        &=
        \tr\left[(\A_2+\De_2)(\Ps^\top\La\Ps)^{-2}\right]\\
        &=
        \tr\left[\A_2(\Ps^\top\La\Ps)^{-2}\right]
        +
        \tr\left[\De_2(\Ps^\top\La\Ps)^{-2}\right]\\
        &=
        \tr\left[\A_2 \A_1^{-2} \left(\I-\De_1(\A_1+\De_1)^{-1}\right)^2\right]
        +
        \tr\left[\De_2(\Ps^\top\La\Ps)^{-2}\right]\\
        &\leq
        \opnorm{(\I-\De_1(\A_1+\De_1)^{-1})^2} \tr\left[\A_2 \A_1^{-2} \right]
        +
        \opnorm{\De_2} \opnorm{(\Ps^\top\La\Ps)^{-2}} \tr\left[\I_N\right]\\
        &\leq
        \left(1+\epsilon\cdot\frac{2}{Q_1}\right)^2\V(M)/\sigma^2
        +
        p_2(M)s_{\min}(\Ps^\top\La\Ps)^{-2}N\\
        &\leq
        \left(1+\epsilon\cdot\frac{2}{Q_1}\right)^2\V(M)/\sigma^2
        +
        \epsilon Q_1^{-2} N,
    \end{align*}
    where we use that fact that $\tr[\M_1\M_2]\leq \opnorm{\M_1}\tr[\M_2]$ for any PDS matrix $\M_1$ in the first inequality. 
    Now set $\epsilon = \min \{ \frac{Q_1}{2}, \frac{1}{Q_1^2N^2} \}$, we have
    \begin{equation*}
        |\V-\V(M)|
        \leq
        3\V(M)+\frac{\sigma^2}{N}.
    \end{equation*}
\end{proof}
In Proposition \ref{proposition:finite_rank}, we can see that for each fixed sample of size $N$, we can find a truncation level $M$ large enough so that the decay of the variance $\V$ is of the same magnitude of $\V(M)$. The extra term does not play an important role in the case of analysing tempered overfitting where $\V=\bigtheta{}{\sigma^2}$ or catastrophic overfitting where $\V\to\infty$.
\newpage
\section{Technical Lemmata} \label{section:technical_lemmata}
This section contains known results from previous work that we use for our main theorems.

\begin{proposition}[Proposition 2.5 in \cite{rudelson2008littlewood}] \label{proposition:rudelson}
    Let $G$ be a $n\times k$ matrix whose entries are independent centered random variables with variances at least 1 and fourth moments bounded by $B$. Let $K\geq1$. Then there exist $C_1,C_2>0$ and $\delta_0\in(0,1)$ that depend only on $B$ and $K$ such that if $k<\delta_0 n$ then
    \begin{equation*}
        \Prob{\inf_{v\in\sph^{k-1}} \eunorm{Gv}\leq C_1 n^{1/2}, \opnorm{G}\leq Kn^{1/2} } \leq e^{-C_2 n}.
    \end{equation*}
    If the random variable is sub-Gaussian, the condition on the operator norm $\opnorm{G}\leq Kn^{1/2} $ can be dropped. 
\end{proposition}

\begin{theorem}[Corollary 5.35 in \cite{vershynin2010introduction}] \label{theorem:vershynin}
    Let $\mathbf{A}$ be an $N\times n$ matrix whose entries are independent standard normal random variables. Then for every $t\geq 0$, with probability at least $1-\Exp{-t^2/2}$, we have
    \begin{equation*}
        \sqrt{N} - \sqrt{n} - t \leq s_{\min}(\mathbf{A})
        \leq s_{\max}(\mathbf{A})
        \leq \sqrt{N} + \sqrt{n} + t.
    \end{equation*}
\end{theorem}

\begin{theorem}[Theorem 5.39 and Remark 5.40 in \cite{vershynin2010introduction}] \label{theorem:vershynin:2}
    Let $\mathbf{A}$ be an $N\times n$ matrix with independent rows $\mathbf{A}_i$ of sub-Gaussian random vector with covariance $\bm{\Sigma}\eqdef\Expect{}{\mathbf{A_i}\mathbf{A_i}^\top}\in\R^{n\times n}$. Then there exists constants $C_3,C_4>0$ (depending only on the sub-Gaussian norm of entries of $\mathbf{A}$), such that for any $t\geq0$, with probability at least $1-2e^{-C_3t^2}$, we have 
    \begin{equation*}
        \opnorm{\frac{1}{N}\mathbf{A}^\top\mathbf{A}-\bm{\Sigma}} \leq \max\{\delta,\delta^2\}\opnorm{\Sigma}.
    \end{equation*}
    where $\delta=C_4\sqrt{\frac{n}{N}}+\frac{t}{N}$. In particular, if $\bm{\Sigma}=\mathbf{I}_n$, we have 
    \begin{equation*}
        \sqrt{N}-\sqrt{C_4n} -t \leq s_{\min}(\mathbf{A}) \leq s_{\max}(\mathbf{A}) \leq \sqrt{N}+\sqrt{C_4n} +t.
    \end{equation*}
\end{theorem}

\begin{theorem}[Theorem 9 (modified) in \cite{koltchinskii2017concentration}] \label{theorem:koltchinskii}
    Let $\A$ be an $N\times n $ matrix whose columns are i.i.d. sub-Gaussian centered random vectors with covariance $\Si$. Then there exists a constant $C>0$, such that, for any $t\geq1$, with probability at least $1-e^{-t}$, it holds that 
    \begin{equation*}
        \opnorm{\frac{1}{n}\A\A^\top - \Si }
        \leq
        C \opnorm{\Si}\min\left\{ \sqrt{\rho}, \rho, \sqrt{\frac{t}{n}}, \frac{t}{n} \right\},
    \end{equation*}
    where $\rho = \frac{\tr[\Si]}{n\opnorm{\Si}}$ is the (re-scaled) effect rank of the covairance $\Si$.
\end{theorem}
\begin{remark}[Dimension-free bound]
Theorem \ref{theorem:koltchinskii} differs from Theorem \ref{theorem:vershynin:2} in that the bound in the former contains both dimensions $N$ and $n$, while the latter only contains $n$.
\end{remark}

\begin{theorem}[Theorem 2.5 in \cite{tsigler2023benign}] \label{theorem:tsigler}
    Suppose Assumption \ref{assumption:sub_gaussian_design} holds. 
    Let $\mathbf{A}_l=\lambda\mathbf{I}_N+\sum_{k=l+1}^M \lambda_k \bm{\psi}_k\bm{\psi}_k^\top\in\R^{N\times N}$. Then there exists a constant $c>0$, such that for any $l<N/c$, with probability of at least $1-ce^{-N/c}$, if $\mathbf{A}_l$ is positive definite, then
    \begin{align*}
        \begin{split}
            \mathcal{B}/c &\leq \|\bm{\theta}_{>l}^*\|_{\bm{\Lambda}_{>l}}^2  \left(1+\frac{s_1(\mathbf{A}_l^{-1})^2}{s_N(\mathbf{A}_l^{-1})^2}+N\lambda_{l+1}s_1(\mathbf{A}_l^{-1})\right)\\
            &\quad+ \|\bm{\theta}_{\leq l}^*\|_{\bm{\Lambda}_{\leq l}^{-1}}^2\left(\frac{1}{N^2s_N(\mathbf{A}_l^{-1})^2}+\frac{\lambda_{l+1}}{N}\frac{s_1(\mathbf{A}_l^{-1})}{s_N(\mathbf{A}_l^{-1})^2}\right)
        \end{split} \\
        \mathcal{V}/c &\leq \frac{s_1(\mathbf{A}_l^{-1})^2}{s_N(\mathbf{A}_l^{-1})^2}\frac{l}{N} + Ns_1(\mathbf{A}_l^{-1})^2\sum_{k>l}\lambda_k^2,
    \end{align*}
    where  $\bm{\theta}^*=\bm{\theta}_{\leq l}^*\dirsum\bm{\theta}_{>l}^*$ is the splitting of the target function coefficient; and $\|\mathbf{v}\|_\mathbf{M}\eqdef \sqrt{\mathbf{v}^\top\mathbf{M}\mathbf{v}}$ for any vector $\mathbf{v}$ and matrix $\mathbf{M}$ with appropriate dimension.
\end{theorem}


\begin{theorem}[Lemma 7 and Theorem 10 in \cite{tsigler2023benign}] \label{theorem:tsigler:2}
    Suppose sub-Gaussian design assumption \ref{assumption:sub_gaussian_design} holds. In addition, fix constants $A>0, B>\frac{1}{N}$ and suppose either (i) the (normalized) effective rank $\rho_l\eqdef\frac{1}{N\lambda_{l+1}}\sum_{k=l+1}^M\lambda_k\in (A,B)$; or (ii) $l=\min\{\ell:\rho_\ell>B\}$. Then there exists a constant $C,C'$, such that if $l<N/C$, with a probability at least $1-Ce^{-N/C}$, we have 
    \begin{align*}
        \mathcal{V} \geq C' \left( \frac{l}{N} + \frac{N\sum_{k>l}\lambda_k^2}{\left(\sum_{k>l}\lambda_k\right)^2} \right).
    \end{align*}
\end{theorem}

\begin{lemma}[Negative second moment identity, Exercise 2.7.3 in \cite{tao2012topics}] \label{lemma:negative_second_moment}
    Let $\mathbf{M}$ be an invertible $n\times n$ matrix, let $\mathbf{R}_1,...,\mathbf{R}_n$ be the rows of $\mathbf{M}$ and let $\mathbf{C}_1,...,\mathbf{C}_n$ be the columns of $\mathbf{M}^{-1}$. For each $1\leq i\leq n$, let $\mathbf{N}_i$ be a unit normal vector orthogonal to the subspace spanned by the all rows $\mathbf{R}_1,...,\mathbf{R}_n$ except $\mathbf{R}_i$. Then we have 
    \begin{equation*}
        \eunorm{\mathbf{C}_i}^2 = (\mathbf{R}_i^\top\mathbf{N}_i)^{-2}
        \text{ and }
        \sum_{i=1}^n s_i(\mathbf{M})^{-2} = \sum_{i=1}^n  (\mathbf{R}_i^\top\mathbf{N}_i)^{-2}.
    \end{equation*}
\end{lemma}
\begin{proof}
    Note that $\mathbf{R}_i^\top\mathbf{C}_j=\delta_{ij}$ and the rows $\mathbf{R}_i$'s spans the space $\R^N$. Hence we have $\mathbf{C}_i=\pm\eunorm{\mathbf{C}_i}\mathbf{N}_i$ for all $i$ and $\eunorm{\mathbf{C}_i}^2= (\mathbf{R}_i^\top\mathbf{C}_i/\mathbf{R}_i^\top\mathbf{N}_i)^2 = (\mathbf{R}_i^\top\mathbf{N}_i)^{-2}$ which proves the first statement. For the second statement, note that 
    \begin{equation*}
        \sum_{i=1}^n \lambda_{i}(\mathbf{M})^{-2} 
        =\sum_{i=1}^n \lambda_{i}(\mathbf{M}^{-1})^{2} 
        =\tr[(\mathbf{M}^{-1})^\top(\mathbf{M}^{-1})]
        = \sum_{i=1}^n \eunorm{\mathbf{C}_i}^2
        = \sum_{i=1}^n  (\mathbf{R}_i^\top\mathbf{N}_i)^{-2}.
    \end{equation*}
\end{proof}
\begin{lemma}[lower bound of $s_{\min}$] \label{lemma:s_min:lb}
    $\mathbf{K}_N=\sum_{k=1}^N \lambda_k\bm{\psi}_k\bm{\psi}_k^\top \prec \mathbf{K}$. Let $\bm{\Lambda}_N=\diag(\lambda_k)_{k=1}^N\in\R^{N\times N}$ and $\bm{\Psi}_N=(\bm{\psi}_k)_{k=1}^N\in\R^{N\times N}$ and set $\mathbf{M}=\bm{\Lambda}_N^{1/2}\bm{\Psi}_N$ which is invertible almost surely. Note that $\mathbf{K}_N=\mathbf{M}^\top\mathbf{M}$.
    Let $\mathbf{R}_1,...,\mathbf{R}_n$ be the rows of $\mathbf{M}$ and let $\mathbf{C}_1,...,\mathbf{C}_n$ be the columns of $\mathbf{M}^{-1}$. For each $1\leq i\leq n$, let $\mathbf{N}_i$ be a unit normal vector orthogonal to the subspace spanned by the all rows $\mathbf{R}_1,...,\mathbf{R}_n$ except $\mathbf{R}_i$. we have
    \begin{equation*}
        s_{\min}(\mathbf{K}) \geq  \frac{\lambda_N}{\sum_{k=1}^N\frac{\lambda_N}{\lambda_k}(\bm{\psi}_k^\top\mathbf{N}_k)^{-2}}.
    \end{equation*}
\end{lemma}
\begin{proof}
    Since $s_{\min} \geq s_N(\mathbf{K}_N)$, WLOG: assume $M=N$. Then by Lemma \ref{lemma:negative_second_moment}, 
    \begin{equation*}
        s_N(\mathbf{K}_N)^{-1}
        \leq \sum_{k=1}^N s_k(\mathbf{K}_N)^{-1} 
        = \sum_{k=1}^N s_k(\mathbf{M})^{-2} 
        = \sum_{k=1}^N \left( \sqrt{\lambda_k}\bm{\psi}_k^\top \mathbf{N}_k  \right)^{-2},
    \end{equation*}
    where $\mathbf{N}_k$ denote a unit normal vector orthogonal to the subspace spanned by the all rows $\mathbf{R}_1,...,\mathbf{R}_n$ of $\mathbf{M}$ except $\mathbf{R}_i$. Hence
    \begin{equation} \label{line:smallest_singular_value_lower_bound}
        s_{\min} 
        \geq s_N(\mathbf{K}_N) 
        \geq \frac{\lambda_N}{\sum_{k=1}^N\frac{\lambda_N}{\lambda_k}(\bm{\psi}_k^\top\mathbf{N}_k)^{-2}}.
    \end{equation}
\end{proof}

\begin{lemma}[Anti-Concentration Result For Gaussian Laws] \label{lemma:anti_concentration_gaussian}
    Let $g$ be a standard Gaussian variable, then
    \begin{equation} \label{line:anticoncentration}
        \Prob{|g|\leq t} \leq \frac{2t}{\sqrt{2\pi}},\ \forall t\geq0.
    \end{equation}
\end{lemma}

\begin{lemma}[Sub-Exponential Deviation, see Corollary 5.17 in \cite{vershynin2010introduction}] \label{lemma:sub_exponential_deviation}
    Let $N\in\N$. Let $X_1,...,X_N$ be independent centered random variables with sub-exponential norms bounded by $B$. Then for any $\delta>0$,
    \begin{equation*}
        \Prob{|\sum_{i=1}^NX_i|>\delta N} \leq 2\exp\left(-C_5\min\left\{\frac{\delta^2}{B^2},\frac{\delta}{B}\right\}N\right),
    \end{equation*}
    where $C_5>0$ is an absolute constant.
    
    In particular, if $X\sim\chi(N)$ is the Chi-square distribution, then
    $\Prob{|\frac{X}{N}-1|>t}\leq 2e^{-Nt^2/8},\ \forall t\in(0,1)$.
\end{lemma}

\begin{theorem}[Theorem 1.1 in \cite{rudelson2009smallest} /Theorem 5.38 in \cite{vershynin2010introduction}] \label{theorem:rudelson:2}
    Let $\mathbf{A}$ be an $N\times n$ random matrix whose entries are i.i.d. sub-Gaussian random variables with zero mean and unit variance. Then there exists constants $C_6>0,C_7\in(0,1)$ such that for any $\delta>0$,
    \begin{equation*}
        \Prob{s_{\min}(\mathbf{A})\leq \delta (\sqrt{N}-\sqrt{n-1})} \leq (C_6\delta)^{N-n+1}+C_7^N.
    \end{equation*}
    In particular, if $N=n$,
    \begin{equation*}
        s_{\min}(\mathbf{A}) \gtrsim N^{-1/2}
    \end{equation*}
    with high probability.
\end{theorem}

\begin{theorem}[Theorem 5.58 in \cite{vershynin2010introduction}] \label{theorem:vershynin:3}
    Let $\mathbf{A}$ be an $N\times n$ matrix ($N\geq n$) with independent columns $\mathbf{A}_i\in\R^N$ of sub-Gaussian isotropic random vector with with $\eunorm{\mathbf{A}_i}=\sqrt{N}$ almost surely. Then there exists constants $C_8,C_9>0$ (depending only on the sub-Gaussian norm of entries of $\mathbf{A}$), such that for any $t\geq0$, with probability at least $1-2e^{-C_8t^2}$, we have 
    \begin{equation*}
        \sqrt{N}-C_9\sqrt{n} -t \leq s_{\min}(\mathbf{A}) \leq s_{\max}(\mathbf{A}) \leq \sqrt{N}+C_9\sqrt{n} +t.
    \end{equation*}
\end{theorem}

\begin{lemma}[Corollary 4.3.12 in \cite{horn2012matrix}] \label{lemma:horn}
    Let $\M_1,\M_2$ are symmetric matrix. If $\M_2$ is positive semi-definite, then
    \begin{equation*}
        s_{\max}(\M_1) \leq s_{\max}(\M_1+\M_2), \quad s_{\min}(\M_1) \leq s_{\min}(\M_1+\M_2).
    \end{equation*}
\end{lemma}

\begin{lemma}[Variance expression] \label{lemma:variance:expression}
    Recall the definition of the variance $\V \eqdef \mathcal{R} - \B$. In the case of kernel ridgeless regression where $\lambda=0$ and the kernel matrix $\K$ can be written as $\K = \Ps^\top\La\Ps$ by Mercer decomposition, the variance admits the following expression:
    \begin{equation*}
        \V
        =
        \sigma^2 \tr\left[ (\Ps^\top\La^2\Ps)(\Ps^\top\La\Ps)^{-2} \right].
    \end{equation*}
\end{lemma}
\begin{proof}
    By definition,
    \begin{align*}
        \V
        &=
        \mathcal{R} - \B\\
        &=
        \Expect{x,\epsilon}{(f^*(x)-\hat{f}(x))^2}
        -
        \Expect{x,\epsilon}{(f^*(x)-\K_x^\top\K^{-1}f^*(\mathbf{X}))^2}\\
        &=
        \Expect{x,\epsilon}{(\K_x^\top\K^{-1}\ep)^2}\\
        &=
        \Expect{x, \epsilon}{\ep^\top\K^{-1}\K_x\K_x^\top\K^{-1}\ep}\\
        &=
        \Expect{\epsilon}{\ep^\top\K^{-1}\Ps^\top\La^2\Ps\K^{-1}\ep}\\
        &=
         \Expect{\epsilon}{\tr\left[\K^{-1}\Ps^\top\La^2\Ps\K^{-1}\ep\ep^\top\right]}\\
        &=
        \sigma^2 \tr\left[\K^{-1}\Ps^\top\La^2\Ps\K^{-1}\right]\\
        &=
        \sigma^2 \tr\left[(\Ps^\top\La^2\Ps)(\Ps^\top\La\Ps)^{-2}\right].
    \end{align*}
\end{proof}

\end{document}